\documentclass{article}
\usepackage[T1]{fontenc}
\usepackage[utf8]{inputenc}

\usepackage{geometry}
\usepackage{natbib}

\usepackage[unicode=true,pdfusetitle,
bookmarks=true,bookmarksnumbered=false,bookmarksopen=false,
 breaklinks=true,pdfborder={0 0 0},pdfborderstyle={},backref=false,colorlinks=true]{hyperref}
\hypersetup{citecolor=blue} 

\usepackage[utf8]{inputenc}
\usepackage{graphicx}
\usepackage{amsmath}
\usepackage{amssymb}
\usepackage{amsthm}
\usepackage{amsfonts}
\usepackage{mathrsfs}
\usepackage{mathtools}
\usepackage{appendix}
\usepackage{booktabs}
\usepackage{soul}
\usepackage{hyperref}
\usepackage{enumitem}
\usepackage{float}
\usepackage{array}
\usepackage{multirow}
\usepackage[algo2e,ruled]{algorithm2e}

\usepackage{todonotes}
\usepackage[showdeletions]{color-edits}
\addauthor[Matthew]{mz}{red}
\definecolor{myForestGreen}{RGB}{34, 139, 34}
\addauthor[Yudong]{yc}{myForestGreen}
\addauthor[Guy]{gz}{blue}

\usetikzlibrary{backgrounds}
\usetikzlibrary{arrows.meta}

\theoremstyle{plain}
\newtheorem{theorem}{\protect\theoremname}[section]
\theoremstyle{plain}
\newtheorem{lemma}[theorem]{\protect\lemmaname}
\theoremstyle{remark}

\theoremstyle{plain}
\newtheorem{corollary}[theorem]{\protect\corollaryname}
\theoremstyle{plain}

\theoremstyle{definition}
\newtheorem{definition}[theorem]{\protect\definitionname}
\theoremstyle{plain}

\providecommand{\corollaryname}{Corollary}
\providecommand{\lemmaname}{Lemma}
\providecommand{\remarkname}{Remark}
\providecommand{\theoremname}{Theorem}
\providecommand{\conjecturename}{Conjecture}
\providecommand{\definitionname}{Definition}
\providecommand{\propositionname}{Proposition}

\newcommand{\argmax}{\text{argmax}}
\newcommand{\R}{\mathbb{R}}

\renewcommand{\P}{\mathbb{P}}
\newcommand{\E}{\mathbb{E}}
\newcommand{\V}{\mathbb{V}}
\newcommand{\var}{\mathrm{Var}}
\newcommand{\ind}{\mathbb{I}}
\newcommand{\wt}{\widetilde}
\newcommand{\bigo}{O}
\newcommand{\tbigo}{\widetilde{O}}
\newcommand{\zero}{\mathbf{0}}

\newcommand{\St}{\mathcal{S}}
\newcommand{\A}{\mathcal{A}}
\newcommand{\pistar}{\pi^\star}
\newcommand{\vstar}{V^\star}
\newcommand{\vtilstar}{\widetilde{V}^\star}
\newcommand{\vpi}{V^\pi}
\newcommand{\hstar}{h^\star}
\newcommand{\rhostar}{\rho^\star}
\newcommand{\rhopi}{\rho^\pi}
\newcommand{\qstar}{Q^\star}
\newcommand{\qhatstar}{\widehat{Q}^\star}

\newcommand{\reg}{\mathrm{Regret}}

\newcommand{\clipH}{\mathrm{Clip}_H}
\newcommand{\imax}{i_{\mathrm{max}}}

\newcommand{\nleave}{N_{\mathrm{leave}}}
\newcommand{\nstay}{N_{\mathrm{stay}}}

\newcommand{\itersk}{\mathrm{iters}_k}

\newcommand{\stay}{\texttt{stay}}
\newcommand{\leave}{\texttt{leave}}

\usepackage{scalerel,stackengine}
\newcommand\cig[1]{\scalerel*[5.5pt]{\Big#1}{
  \ensurestackMath{\addstackgap[1.5pt]{\big#1}}}}

\global\long\def\spannorm#1{\left\Vert #1\right\Vert _{\textnormal{sp}}}
\global\long\def\bspannorm#1{\bigl\Vert #1\bigr\Vert _{\textnormal{sp}}}

\global\long\def\tspannorm#1{\Vert #1\Vert _{\textnormal{sp}}}

\newcommand{\perstepv}{\V_\gamma^\star}
\newcommand{\varstar}{\mathrm{Var}_\gamma^\star}
\newcommand{\vardiff}{\mathrm{Var}^\mathrm{diff}}

\newcommand{\tmodel}{\mathscr{T}_\mathrm{model}}
\newcommand{\tmart}{\mathscr{T}_\mathrm{mart}}
\newcommand{\tind}{\mathscr{T}_\mathrm{ind}}

\global\long\def\varsa#1{\mathbb{V}_{s,a}\left( #1\right)}
\global\long\def\varstat#1{\mathbb{V}_{s_t,a_t}\left( #1\right)}

\title{Optimal Variance-Dependent Regret Bounds for Infinite-Horizon
MDPs}

\usepackage{authblk}
\author{Guy Zamir}
\author{Matthew Zurek}
\author{Yudong Chen}
\affil{Department of Computer Sciences, University of Wisconsin--Madison\\\vspace{0.8em}
\texttt{gzamir@wisc.edu}\quad\texttt{matthew.zurek@wisc.edu}\quad\texttt{yudongchen@cs.wisc.edu}}

\date{}

\begin{document}

\maketitle

\begin{abstract}
    Online reinforcement learning in infinite-horizon Markov decision processes (MDPs) remains less theoretically and algorithmically developed than its episodic counterpart, with many algorithms suffering from high ``burn-in'' costs and failing to adapt to benign instance-specific complexity. In this work, we address these shortcomings for two infinite-horizon objectives: the classical average-reward regret and the $\gamma$-regret. We develop a single tractable UCB-style algorithm applicable to both settings, which achieves the first optimal variance-dependent regret guarantees. Our regret bounds in both settings take the form $\tbigo( \sqrt{SA\,\text{Var}} + \text{lower-order terms})$, where $S,A$ are the state and action space sizes, and $\text{Var}$ captures cumulative transition variance. This implies minimax-optimal average-reward and $\gamma$-regret bounds in the worst case but also adapts to easier problem instances, for example yielding nearly constant regret in deterministic MDPs. Furthermore, our algorithm enjoys significantly improved lower-order terms for the average-reward setting. With prior knowledge of the optimal bias span $\tspannorm{\hstar}$, our algorithm obtains lower-order terms scaling as $\tspannorm{\hstar}S^2 A$, which we prove is optimal in both $\tspannorm{\hstar}$ and $A$. 
    Without prior knowledge, we prove that no algorithm can have lower-order terms smaller than $\tspannorm{\hstar}^2SA$, and we provide a prior-free algorithm whose lower-order terms scale as $\tspannorm{\hstar}^2S^3A$, nearly matching this lower bound. Taken together, these results completely characterize the optimal dependence on $\spannorm{\hstar}$ in both leading and lower-order terms, and reveal a fundamental gap in what is achievable with and without prior knowledge.
\end{abstract}

\section{Introduction}

We study online reinforcement learning (RL) in tabular Markov decision processes (MDPs), where an agent interacts with an unknown environment and aims to maximize cumulative reward. We specifically consider infinite-horizon continuing settings, where the environment does not contain a built-in reset mechanism. Despite its practical relevance and foundational significance, online infinite-horizon RL is much less well understood theoretically than the finite-horizon episodic setting. In this work we study two particular performance measures for infinite-horizon problems.

The most classical performance objective is the average-reward regret $\sum_t (\rhostar-r_t) $ introduced by the seminal work \citet{auer_near-optimal_2008}, which measures the instantaneous reward $r_t$ of the agent against the optimal gain $\rhostar$, which is the best long-term average reward per timestep of any policy. The reset-less nature of infinite-horizon online RL requires additional structural assumptions to permit sublinear regret bounds, such as for the MDP to be communicating, and these also ensure that $\rhostar$ can be defined independent of an initial state. 
\citet{auer_near-optimal_2008} proved a regret lower bound of $\Omega(\sqrt{DSAT})$, where $D$ is the MDP diameter, $S$ and $A$ are the numbers of states and actions, and $T$ is the time horizon.
Extensive research effort has since been dedicated to matching this regret lower bound and relaxing the communicativity assumptions by replacing $D$ with $\tspannorm{\hstar}$, the maximum gap in cumulative rewards between any two starting states, which is smaller than $D$ and finite even in non-communicating MDPs (e.g., \citealt{bartlett_regal_2012, fruit_efficient_2018, fruit_near_2019, talebi2018variance, zhang_ucbavg_2023}), culminating in the minimax-optimal algorithms of \citet{zhang_regret_2019, boone_achieving_2024}.

Another performance measure for infinite-horizon online RL is the $\gamma$-regret $\sum_t ((1-\gamma)V_\gamma^\star(s_t) - r_t) $ introduced by \citet{liu_regret_2021}, where $V_\gamma^\star(s_t)$ denotes the $\gamma$-discounted optimal value function at the state $s_t$ encountered by the agent. 
While $\gamma$-regret may appear to be a weaker regret notion due to its comparison with the agent's own trajectory rather than that of an optimal policy, this feature has the advantage of enabling sublinear regret bounds without the structural assumptions needed for the average-reward setting. Furthermore, when such communicativity assumptions do hold, we show the $\gamma$-regret can actually be used to control the average-reward regret in an optimal manner, which is the approach we take in this paper (see Lemma~\ref{lem:avg_to_disc_reduction}). Recent work has developed algorithms achieving the minimax-optimal $\gamma$-regret, $\tbigo(\sqrt{SAT/(1-\gamma)})$ \citep{he_nearly_2021, ji_regret-optimal_2023}.

Despite all of the aforementioned algorithmic progress on the average-reward and the $\gamma$-regret settings, there are  several significant limitations of all existing work in both settings relative to episodic online RL.
First, existing minimax-optimal algorithms incur a large burn-in cost, meaning that they only attain the optimal regret rate when $T$ is very large.
For example, the only computationally efficient and minimax-optimal algorithm for the average-reward regret, PMEVI-DT \citep{boone_achieving_2024}, has a regret bound of $\tbigo\cig( \sqrt{\tspannorm{\hstar}SAT} + \tspannorm{\hstar}S^{\frac{5}{2}}A^{\frac{3}{2}}T^{\frac{9}{20}}\cig)$, and thus it only matches the optimal rate for $T\geq \tspannorm{\hstar}^{10}S^{40}A^{20}$.
This contrasts the episodic setting, where significant effort has been expended towards remedying such issues (e.g., \citealt{zhang_is_2021, zhou_sharp_2023, zhang_settling_2024}).
Second, prior work for infinite-horizon settings fails to adapt to easier problem instances such as low-variance or deterministic MDPs, where substantially smaller regret should be possible.
In episodic RL, this gap has been addressed through variance-dependent regret bounds, which interpolate between stochastic and deterministic environments and can be optimal in both regimes \citep{zhou_sharp_2023}.
However, no optimal variance-dependent regret guarantees have been established for either infinite-horizon setting that we consider.

\begin{table}[t]
\label{fig:reg_table}
\centering
\resizebox{\textwidth}{!}{\begin{tabular}{c | c c c c c c}
     Algorithm & $\widetilde{O}(\cdot)$ Regret  & \makebox[1cm]{Priorless?} & \makebox[2cm]{Burn-In Cost} & \makebox[1.2cm]{Tractable?} & Type \\
     \hline
     UCRL2 \citep{auer_near-optimal_2008} & $DS\sqrt{AT}$  & \checkmark & N/A & \checkmark & EVI\\
     REGAL \citep{bartlett_regal_2012} & $\tspannorm{\hstar}S\sqrt{AT}$  & $\times$ & N/A & $\times$ & EVI \\
     SCAL \citep{fruit_efficient_2018} & $\tspannorm{\hstar}S\sqrt{AT}$  & $\times$ & N/A & \checkmark & EVI \\
     KL-UCRL \citep{talebi2018variance} & $\sqrt{ST\sum_{s,a}\V^\star_{s,a}} + D\sqrt{T}$  & \checkmark & N/A & \checkmark & EVI \\
     EBF \citep{zhang_regret_2019} & $\sqrt{\tspannorm{\hstar}SAT}$  & $\times$ & (*) & $\times$ & EVI \\
     UCB-AVG \citep{zhang_ucbavg_2023} & $S^5A^2\tspannorm{\hstar}\sqrt{T}$ & $\times$ & N/A & \checkmark & UCB \\
     PMEVI-DT \citep{boone_achieving_2024} & $\sqrt{\tspannorm{\hstar}SAT}$  & \checkmark & $\tspannorm{\hstar}^{10} S^{40} A^{20}$ & \checkmark & EVI \\
     $\gamma$-UCB-CVI \citep{hong_reinforcement_2025} & $\tspannorm{\hstar}S\sqrt{AT}$ & $\times$ & N/A & \checkmark & UCB \\
     \hline
     Corollary \ref{cor:regret_prior} & $\sqrt{\mathrm{Var}^\star_{1-1/T}SA} + \Gamma \tspannorm{\hstar}SA$  & $\times$ & $\tspannorm{\hstar}S^3A$ & \checkmark & UCB \\
    Corollary \ref{cor:regret_no_prior}  & $\sqrt{\tspannorm{\hstar}SAT}$  & \checkmark & $\spannorm{\hstar}^2S^3A$ & \checkmark & UCB \\
     \hline
     \hline
     \textbf{Lower Bound} \citep{auer_near-optimal_2008} & \multicolumn{4}{c}{$\sqrt{DSAT}$, implies $\sqrt{\tspannorm{h^\star}SAT}$}
\end{tabular}}
\caption{\textbf{Comparison of algorithms and regret bounds for average-reward MDPs.} 
Here $\tspannorm{\hstar}$ is the span of the optimal bias function and $D$ is the diameter, which satisfy $\tspannorm{\hstar} \leq D$. We always have $\Gamma \leq S$, and $\Gamma=1$ in deterministic MDPs. 
$\mathrm{Var}^\star_{1-1/T}$ and $\V^\star_{sa}$ are instance-dependent variance parameters, which in particular are $0$ for deterministic MDPs. 
Also $\mathrm{Var}^\star_{1-1/T} \leq \widetilde{O}(\tspannorm{h^\star}T + \tspannorm{\hstar}^2)$. 
Only the leading terms as $T \to \infty$ of the regret bound are shown. 
The burn-in cost is defined as the smallest $T$ for which the algorithm achieves a minimax-optimal regret of $\widetilde{O}(\sqrt{\tspannorm{h^\star}SAT})$, or N/A if this does not occur. 
Priorless means that an algorithm does not require prior knowledge  about the value of $\tspannorm{h^\star}$. (*) The burn-in cost of EBF \citep{zhang_regret_2019} is $\tspannorm{\hstar}^6S^{4}A^4+\tspannorm{\hstar}^6 S^6 A^2+ \tspannorm{\hstar}^3 S^{12} A^3$.
}
\end{table}

\subsection{Contributions}

In this paper, we establish the first optimal variance-dependent regret guarantees for infinite-horizon MDPs. 
Our main contribution is a single tractable algorithm that, in both the average-reward and $\gamma$-regret settings, attains a regret bound of the form
\begin{equation*} 
    \tbigo\big( \sqrt{\text{Var}_\gamma\ SA} + \text{lower-order terms} \big).
\end{equation*}
We handle both of these infinite-horizon settings in a unified way by treating $\gamma = 1-\frac{1}{T}$ as a tuning parameter in the average-reward case.
Here $\text{Var}_\gamma$ is a cumulative variance term that measures the stochasticity of the transition dynamics along the learner's trajectory.
When the MDP is deterministic, we have $\text{Var}_\gamma=0$, and thus the resulting regret is independent of $T$ up to logarithmic factors. We also show $\text{Var}_\gamma \leq  \tbigo(\tspannorm{\vstar_\gamma}T + \tspannorm{\vstar_\gamma}^2)$, so by using that $\tspannorm{\vstar_\gamma} \leq 2\tspannorm{\hstar}$ in weakly communicating MDPs \citep{wei_model-free_2020}, or simply that $\tspannorm{\vstar_\gamma} \leq \frac{1}{1-\gamma}$, we obtain minimax-optimal regret bounds in both the average-reward and $\gamma$-regret settings, respectively.

Focusing on the average-reward setting, another main contribution is that we significantly improve the lower-order terms relative to prior work.
When given prior knowledge of the bias span $\tspannorm{\hstar}$, our (variance-dependent and minimax-optimal) result contains lower-order terms scaling as $\tspannorm{\hstar}S^2A$, and we show via matching lower bounds that this dependence on $\tspannorm{\hstar}$ and $A$ is unimprovable.
Additionally, without prior knowledge of $\tspannorm{\hstar}$, we obtain lower-order terms scaling as $\tspannorm{\hstar}^2S^3A$.
We also show a surprising hardness result that no algorithm without prior knowledge of $\spannorm{\hstar}$ can obtain lower-order terms better than $\tspannorm{\hstar}^2SA$. 
Taken together, our results nearly completely characterize the optimal dependence on $\tspannorm{\hstar}$ in both the leading and lower-order terms and reveal a fundamental separation between what is achievable with and without prior knowledge.

To obtain our improved bounds for $\gamma$-regret, we develop a model-based, upper-confidence-bound(UCB)-based algorithm called Fully Optimizing Clipped UCB Solver (FOCUS).
We improve upon previous UCB-based algorithms for $\gamma$-regret by incorporating a sharp Bernstein-style bonus and span-clipping into our empirical Bellman operator.
Crucially, instead of performing a single step of value iteration at each update, FOCUS repeatedly applies the empirical Bellman operator until convergence.
This design ensures that the Q-estimates fully exploit the collected data at each update and is essential for obtaining variance-dependent bounds. 
Finally, FOCUS is the first UCB-style algorithm to achieve minimax-optimal regret guarantees for the online average-reward setting, contrasting previous optimal algorithms which instead depend on extended value iteration (EVI).

\subsection{Related Work}
Here we discuss related work, and we also refer to Appendix \ref{sec:related_work} for more discussion.
\paragraph{Online Average-Reward}
The average-reward setting is classical and well-studied for online infinite-horizon RL.
See Table \ref{fig:reg_table} for a comparison of important prior work; however, given the extensive history of this problem, this table is non-exhaustive.
The seminal work of \cite{auer_near-optimal_2008} introduces UCRL2 and establishes a regret bound of $\tbigo(DS\sqrt{AT})$, along with a lower bound of $\Omega(\sqrt{DSAT})$.
\cite{bartlett_regal_2012} establish regret bounds that depend on $\tspannorm{\hstar}$ instead of $D$, but their algorithm REGAL is computationally intractable algorithm and requires prior knowledge of $\tspannorm{\hstar}$.
The computationally effcient algorithm SCAL by \cite{fruit_efficient_2018} utilizes the span-clipping technique to match the bound of REGAL.
The EBF algorithm of \cite{zhang_regret_2019} is the first to achieve minimax optimality of $\tbigo(\sqrt{\tspannorm{\hstar}SAT})$, but it is intractable and requires prior knowledge.
\cite{boone_achieving_2024} resolves these issues by developing PMEVI-DT, which is tractable, minimax optimal and prior-knowledge-free.
\cite{talebi2018variance} obtain variance-aware guarantees for the KL-UCRL algorithm, but their bounds remain suboptimal in the worst-case and suffer $\sqrt{T}$ dependence even on deterministic MDPs.
All aforementioned algorithms are based on EVI. UCB-based algorithms, which like ours also employ discounting to approximate the average-reward problem, include
\citet{wei_model-free_2020, zhang_ucbavg_2023, hong_reinforcement_2025}.

\paragraph{$\gamma$-regret}
The $\gamma$-regret notion is introduced by \cite{liu_regret_2021}, who prove the first upper bound for this framework with their Double Q-learning algorithm.
The subsequent work of \cite{he_nearly_2021} proposes UCBVI-$\gamma$, a model-based algorithm with Bernstein-style bonuses, and achieves a $\gamma$-regret bound that is minimax-optimal in the leading term.
\cite{ji_regret-optimal_2023} develop Q-SlowSwitch-Adv, a model-free algorithm that matches the optimal leading term while improving lower-order dependence on $SA$.
More recently, \cite{ma2026eubrl} take a Bayesian approach with the EUBRL algorithm, which leverages epistemic uncertainty for directed exploration and achieves state-of-the-art lower-order dependence on $\frac{1}{1-\gamma}$.
We note that the definition of $\gamma$-regret in our paper matches that of \cite{liu_regret_2021}, whereas \cite{he_nearly_2021}, \cite{ji_regret-optimal_2023}, and \cite{ma2026eubrl} use slightly different definitions.
While there is not an immediate translation between these definitions, they are closely related and algorithms with good guarantees for one should have good guarantees for the others.
For further discussion on this issue, see Appendix A.2 of \cite{he_nearly_2021} and Appendix A of \cite{ji_regret-optimal_2023}.

\section{Preliminaries}

\paragraph{MDP Basics}

A Markov Decision Process is a tuple $(\St, \A, P, r, \mu_0)$\footnote{Sometimes we consider classes of MDPs with the same $\St,\A,r$, and $\mu_0$. In this case we simply denote the MDP by its transition kernel $P$.}, where $\St$ is the state space, $\A$ is the action space, $P \colon \St \times \A \to \Delta(\St)$ is the transition kernel with $\Delta(\St)$ denoting the probability simplex over $\St$, $r\colon \St \times \A \to [0,1]$ is the reward function, and $\mu_0\in \Delta(\St)$ is the initial state distribution.
We assume $\St$ and $\A$ are finite sets with $S \coloneq |\St|$ and $A \coloneq |\A|$.
A (stationary) policy is a mapping $\pi : \St \to \Delta(\A)$.
When $\pi$ is a deterministic policy, we treat $\pi$ as a mapping $\St \to \A$.
Let $\Pi$ be the set of all stationary deterministic policies.
An initial state distribution $\mu \in \Delta(\St)$ and a policy $\pi$ induce a distribution over trajectories $(s_0, a_0, s_1,a_1,\dots)$, where $s_0 \sim \mu$, $a_t \sim \pi(s_t)$, and $s_{t+1} \sim P(\cdot|s_t,a_t)$.
We let $\mathbb{E}^\pi_{s}$ denote the expectation with respect to this distribution when $\mu$ satisfies $\mu(s)=1$.

For a policy $\pi$ and discount factor $\gamma \in (0,1)$, the discounted value function $V^\pi_\gamma \in [ 0, \frac{1}{1-\gamma} ]^\St$ is defined by $\vpi_\gamma(s) = \mathbb{E}^\pi_s[ \sum_{t=0}^\infty \gamma^t r_t ]$, where $r_t = r(s_t,a_t)$.
Also define the optimal value function $\vstar_\gamma \in [ 0, \frac{1}{1-\gamma}]^\St$ by $\vstar_\gamma(s) = \sup_{\pi\in\Pi} \vpi_\gamma(s)$.
We often write $P_{s,a}$ to denote the row vector such that $P_{s,a}(s') = P(s'|s,a)$.
Then for any $V\in\R^\St$ we have $P_{s,a}V=\mathbb{E}_{s'\sim P(\cdot|s,a)}[V(s')]$.
The gain of a policy $\pi$, $\rho^\pi \in [0,1]^\St$, is defined by $\rhopi(s) = \lim_{T \to \infty} \frac{1}{T}\E^\pi_s[\sum_{t=0}^{T-1} r_t]$.
We define the optimal gain $\rhostar = \sup_{\pi \in \Pi} \rho^\pi$.
The bias function of a policy $\pi$, $h^\pi \in \R^{\St}$, is $h^\pi(s) = \text{C-lim}_{T \to \infty} \E_s^\pi [\sum_{t=0}^{T-1} (r_t - \rho^\pi(s_t))]$. 
The optimal bias function is $\hstar \coloneq h^{\pistar}$, where $\pistar$ is a Blackwell-optimal policy, which satisfies $V_\gamma^{\pistar} = V_\gamma^\star$ for all sufficiently large $\gamma$.

The diameter of an MDP is $D = \max_{s_1\neq s_2} \inf_{\pi\in\Pi} \mathbb{E}_{s_1}^\pi[\eta_{s_2}]$, where $\eta_s$ denotes the hitting time of a state $s\in\St$.
An MDP is (strongly) communicating if its diameter is finite; that is, any state is reachable from any other state under some policy.
An MDP is weakly communicating if the states can be partitioned into two disjoint subsets $\St = \St_1 \cup \St_2$ such that all states in $\St_1$ are transient under all policies, and within $\St_2$ any state is reachable from any other state under some policy.
In weakly communicating MDPs, the optimal gain $\rhostar$ is constant and thus, with a slight abuse of notation, treated as a scalar.
We say an MDP is deterministic if the transition probabilities $P(\cdot|s,a)$ are one-hot vectors, i.e., from each state-action pair the agent will transit to a certain state with probability 1.
We also define $\Gamma \coloneq \max_{(s,a)\in\St\times\A} | \mathrm{supp}(P(\cdot|s,a)) |$.

\paragraph{Online RL and Regrets}
The learner interacts with the MDP for $T$ steps, starting from a state $s_1\sim \mu_0$. 
At each step $t = 1, \dots, T$, the learner at state $s_t$ chooses an action $a_t$ and observes the next state $s_{t+1} \sim P(\cdot | s_t,a_t)$.
The learner aims to maximize the total reward $\sum_{t=1}^T r(s_t,a_t)$ it receives.
We consider two different notions of regret, which measure the disparity between the learner's reward and that of an optimal policy.
Given a discount factor $\gamma \in (0,1)$, define
$
\reg_\gamma(T) \coloneq \sum_{t=1}^T \left( (1-\gamma) \vstar_\gamma(s_t) - r(s_t,a_t) \right).
$
When the  MDP is weakly communicating, further define
$
\reg(T) \coloneq \sum_{t=1}^T \big( \rho^\star - r(s_t,a_t) \big).
$
Throughout the paper, $\gamma$-regret or discounted setting refers to $\reg_\gamma(T)$, and regret or average-reward setting refers to $\reg(T)$.

To be precise, the regret and $\gamma$-regret are both functions of the underlying MDP and the trajectory $s_1,a_1,\dots,s_T,a_T$.
The trajectory is random, with a distribution determined by the MDP, the learning algorithm, and the time horizon $T$.
Oftentimes we only explicitly write $T$ as a parameter in regret because the MDP and learning algorithm are usually clear from context.
In situations where the underlying MDP $P$ and learning algorithm $\texttt{Alg}$ are not obvious, we write $\reg(T, P, \texttt{Alg})$.

\paragraph{Burn-In Cost}

We say that an algorithm achieves a burn-in cost of $f(\tspannorm{\hstar},S,A)$ if for all MDPs $M$ and any $T\geq f(\tspannorm{\hstar_M}, S_M, A_M)$, the regret can be bounded (with high probability) by $\tbigo\Big(\sqrt{\tspannorm{\hstar_M}S_M A_M T}\Big)$, where $\hstar_M, S_M$, and $A_M$ are the optimal bias function, state space size, and action space size, respectively, of $M$.
We use lower-order terms to refer to any of the additive terms in the regret bound besides the leading term, which is typically $\sqrt{\smash[b]{\text{Var}^\star_{1-1/T}SA}}$ or $\sqrt{\tspannorm{\hstar}SAT}$.
Note that the lower-order terms may sometimes dominate the regret, including when the MDP is deterministic so that $\text{Var}_{1-1/T} = 0$.

To illustrate the difference between burn-in cost and lower-order terms, consider an algorithm that obtains a regret bound of $\tbigo(\sqrt{\tspannorm{\hstar}SAT} + \tspannorm{\hstar}S^2A)$.
Then the algorithm has a lower-order term of $\tspannorm{\hstar}S^2A$ and achieves a burn-in cost of $\tspannorm{\hstar}S^3A$.
There is an conversion --- albeit a lossy one --- between the two notions.
If an algorithm has a burn-in cost of $f$, then its regret can be bounded by $\tbigo(\sqrt{\tspannorm{\hstar}SAT} + f(\tspannorm{\hstar},S,A))$, because for large $T$ the leading term dominates, and for small $T$ the regret is at most $T<f(\tspannorm{\hstar},S,A)$.
If an algorithm achieves a regret of $\tbigo(\sqrt{\tspannorm{\hstar}SAT} + f(\tspannorm{\hstar},S,A))$, then the algorithm has a burn-in cost of $\frac{f(\tspannorm{\hstar},S,A)^2}{\tspannorm{\hstar}SA}$, since for $T$ larger than this quantity the leading term dominates.

\paragraph{Additional Notation}

Let $[m]$ denote $\{1,\dots,m\}$ for any positive integer $m$.
Let $\zero, \mathbf{1}$ be the all-zero and all-one vectors.
For $x\in\R^\St$, define the span semi-norm $\tspannorm{x} \coloneq \max_{s\in\St} x(s) - \min_{s\in\St} x(s)$.
For $x,y\in\R^n$, we define $\V(x,y)\coloneq \sum_{i=1}^n x_i y_i^2 - \left(\sum_{i=1}^n x_iy_i\right)^2$. Observe that when $x$ is a probability vector, $\V(x,y)$ is the variance of a random variable that takes value $y_i$ with probability $x_i$.
For $H\geq 0$, we define the clipping operator $\clipH : \R^\St \to \R^\St$ by $(\clipH(V))(s) = \min\{ V(s), \min_{s'\in\St} V(s') + H \}$.
We also define the maximum operator $M \colon \R^{\St\times\A} \to \R^\St$ by $(MQ)(s) = \max_{a\in\A} Q(s,a)$.
The notation $O(\cdot)$ hides constant factors.
The notation $\tbigo(\cdot)$ hides constant factors and possible $\text{polylog}$ factors of $S,A,T$, and $\frac{1}{\delta}$.

\section{Main Results}

In this section, we first describe our algorithm and discuss improvements over prior work.
We then define our variance-dependent term $\varstar$ and relate it to other relevant quantities.
Next, we provide our main bound for $\gamma$-regret and implications for the discounted setting.
We then reduce the average-reward setting to the discounted setting and derive optimal bounds on the regret with and without prior knowledge.
Finally, we state lower bounds showing that our lower-order dependence on $\spannorm{\hstar}$ is optimal in both the prior knowledge and prior-free settings.

\subsection{Algorithm} \label{sec:algorithm}

\begin{algorithm2e}[h]
\caption{Fully Optimizing Clipped UCB Solver (FOCUS)} 
\label{alg:oovi}
\SetAlgoLined
\DontPrintSemicolon
\LinesNumbered
\KwIn{run time $T\geq 1$, discount factor $\gamma\in (0,1)$, failure probability $\delta\in (0,1)$, span clipping parameter $H\geq 1$}

$k \gets 1$, $U \gets \log\left( \frac{1}{\delta'} \right)$, where $\delta' = \delta / \left( 9 S^2 A T \right)$\;
$N(s,a) \gets 0, N(s,a,s') \gets 0$ for all $(s,a,s')\in\St\times\A\times\St$\;
$\widehat{Q}_1(s,a) \gets \frac{1}{1-\gamma}$ for all $(s,a)\in\St\times\A$ \;
Observe $s_1$\;

\For{$t=1,\dots,T$}{
    Take action $a_t \in \argmax_{a\in\A} \widehat{Q}_k(s_t,a)$ and observe $s_{t+1}$ \;
    $N(s_t,a_t) \gets N(s_t,a_t)+1, N(s_t,a_t,s_{t+1}) \gets N(s_t,a_t,s_{t+1})+1$ \;

    \If{$N(s_t,a_t) = 2^i$ for some integer $i\geq 0$} {
        $k \gets k+1$, $\varepsilon_k \gets \frac{1}{t(1-\gamma)}$ \;
        $N_k(s,a) \gets N(s,a), N_k(s,a,s') \gets N(s,a,s')$ for all $(s,a,s')$ \;
        $\widehat{P}^k_{s,a,s'} \gets \frac{N_k(s,a,s')}{N_k(s,a)}$ for all $(s,a,s')$ such that $N_k(s,a) > 0$ \;
        $\widehat{P}^k_{s,a,s'} \gets \frac{1}{S}$ for all $(s,a,s')$ such that $N_k(s,a) = 0$ \;
        Compute $\widehat{Q}_k = \widehat{\mathcal{T}}_k^{(m)}(\zero)$ where $m = \left\lceil \frac{1}{1-\gamma} \log \frac{1+ 32HU}{\varepsilon_k(1-\gamma)} \right\rceil$
    }
}
\end{algorithm2e}

We present our algorithm, Fully Optimizing Clipped UCB Solver (FOCUS), in Algorithm~\ref{alg:oovi}.
The algorithm uses a model-based approach: it keeps track of state-action visitation counts and uses these to maintain an empirical estimate of the transition kernel.
The algorithm runs through episodes, with a new episode starting when the number of visits to a state-action pair doubles.
At the start of the $k$th episode, the algorithm updates the empirical transition kernel $\widehat{P}^k$, then uses a clipped optimistic value iteration procedure to compute $\widehat{Q}_k$, which is an optimistic estimate of $Q^\star$.
At each time step $t$ through the remainder of the episode, the algorithm observes state $s_t\in\St$ and takes a greedy action $a_t \in \argmax_{a\in\A} \widehat{Q}_k(s_t,a)$. 

FOCUS uses the following clipped optimistic empirical Bellman operators $\widehat{\mathcal{T}}_k$.  Fixing the episode number $k$,
for any $Q\in \R^{\St\times\A}$, define
\begin{align*}
    \widehat{\mathcal{T}}_k(Q)(s,a) \coloneq r(s,a) + \gamma \widehat{P}^k_{s,a} \big(\text{Clip}_H (MQ)\big) + \gamma b\big(s,a,\text{Clip}_H (MQ)\big)
\end{align*}
where
$
	b(s,a,V) \coloneq \max\Big\{4\sqrt{\frac{\V\left( \widehat{P}^k_{s,a}, V \right)U}{\max\left\{ N_k(s,a),1 \right\}}}, 32 \frac{HU}{\max\left\{ N_k(s,a),1 \right\}} \Big\}
$
for  $V\in\R^\St$, and $U, H, \gamma$ are defined within Algorithm~\ref{alg:oovi}.
We now discuss the key features of $\widehat{\mathcal{T}}_k$.
First, it uses span-clipping, which ensures that all value estimates have span bounded by a parameter $H$.
Second, the operator incorporates a sharp, Bernstein-style bonus, similar to that of the MVP algorithm \citep{zhang_is_2021}.
Lastly, rather than performing a one-step value iteration as in UCBVI-$\gamma$ \citep{he_nearly_2021} or $\gamma$-UCB-CVI \citep{hong_reinforcement_2025}, FOCUS fully solves for the fixed point of $\widehat{\mathcal{T}}_k$ by iteratively applying it until convergence.
This final ingredient is essential for obtaining a bound tight enough to perform our average-to-discount reduction, for reasons that we further discuss in Section~\ref{sec:technical_highlights}.

\paragraph{Computational Complexity}

We note that FOCUS is computationally tractable because we design the empirical bellman operators $\widehat{\mathcal{T}}_k$ to be $\gamma$-contractions.
Consequently, the iterates converge geometrically to the fixed point of the operator and only $\tbigo(\frac{1}{1-\gamma})$ iterations suffice.
Additionally, although the doubling trick helps replace a factor of $T$ with $SA\log(T)$ in the computational complexity, its main purpose is to remove a factor of $\frac{1}{1-\gamma}$ from the regret bound for a technical reason that we explain further in Section~\ref{sec:technical_highlights}.

We now consider the computational complexity of FOCUS.
By the doubling rule, there are at most $O(SA\log T)$ episodes.
At the start of each episode, we update our empirical counts and transition kernel, which takes $\bigo\left(S^2A\right)$ time.
We then run $\smash{O(\frac{1}{1-\gamma}\log T)}$ steps of value iteration, and since each step of value iteration takes $\bigo (S^2A )$ time, this takes a total of $\smash{O(\frac{S^2A}{1-\gamma}\log T )}$ time.
So across all episodes, the total run time of FOCUS is $\smash{O\big(\frac{S^3A^2}{1-\gamma}(\log T)^2\big)}$.
For the average-reward setting where we set $\gamma = 1 - \frac{1}{T}$, the run time is $O(S^3A^2T)$.

\subsection{Variance Parameters}

Next, we introduce the main variance-dependent term that will later appear in our regret bounds.
\begin{definition}
    Let $\gamma \in (0,1)$.
    For a particular MDP with transition kernel $P$ and trajectory $\{s_t,a_t\}_{t\in[T]}$, define the cumulative variance as 
    \[
    \varstar
    \coloneq 
    \sum\nolimits_{t=1}^T \V\big(P_{s_t,a_t}, \vstar_\gamma\big).
    \]
\end{definition}
One can interpret $\varstar$ as a measure of stochasticity in an MDP.
When the transition probabilities $P(\cdot | s,a)$ are more concentrated on states of similar value, $\varstar$ will generally be smaller.
If the MDP is completely deterministic, it is easy to see that $\varstar = 0$ for any trajectory.
For a more detailed discussion on related variance parameters (in the episodic setting), see \cite{zhou_sharp_2023}.

Existing variance-dependent guarantees for infinite-horizon MDPs (i.e., \citealt{talebi2018variance}) involve terms similar to $T \perstepv $, where $\perstepv \coloneq \max_{s,a} \V(P_{s,a}, \vstar_\gamma)$ is the maximum per-step variance.
Our bounds will instead depend on $\varstar$, which is a random but much sharper quantity.
It is easy to see that $\varstar \leq T \perstepv$.
Furthermore, while $T \perstepv$ can be as large as $T \tspannorm{\vstar_\gamma}^2$, the following lemma shows that $\varstar$ scales with $T \tspannorm{\vstar_\gamma}$. Thus this lemma can be used to derive minimax-optimal span-based bounds from our variance-dependent bounds in both the discounted and average-reward settings, extending a successful approach from offline settings (e.g., \cite{zurek_span-based_2025}) to the variance measure $\varstar$ relevant for online learning.
We defer the proof to Appendix~\ref{sec:var_star_bound_proof}.
\begin{lemma} \label{lem:var_star_bound}
	For $\delta\in(0,1)$, we have with probability $1-\delta$ that
    $
    \varstar \leq O\Big(\tspannorm{\vstar_\gamma}T + \tspannorm{\vstar_\gamma}^2\log(T/\delta)\Big).
    $
\end{lemma}

\subsection{Main Results for Discounted Setting} \label{discounted_main_results}

We now present our main theorem on the performance of FOCUS in the discounted setting.
\begin{theorem}[Variance-Dependent $\gamma$-Regret Bound] \label{thm:var_dependent_gamreg}
	Let $T\geq 1, \gamma \in (0,1), \delta\in(0,1)$.
	For any $H \geq \tspannorm{\vstar_\gamma}$, Algorithm~\ref{alg:oovi} with input $(T,\gamma,\delta,H)$ achieves, with probability at least $1-\delta$,
	\[
	    \reg_\gamma(T) \leq \tbigo\Big( \sqrt{SA \varstar} + \Gamma HSA \Big).
	\]
\end{theorem}
We provide a complete proof in Appendix~\ref{sec:var_dependent_gamreg_proof}.
Theorem~\ref{thm:var_dependent_gamreg} establishes the first variance-dependent $\gamma$-regret bound.
Unlike prior works that fail to exploit easier environments and thereby necessarily scale with $\sqrt{T}$, the leading term of our bound depends on $\varstar$, which as previously mentioned captures the stochasticity of transition dynamics in the MDP.
Consequently, our bound interpolates between stochastic and deterministic environments and is significantly sharper in the latter case.

To illustrate these improvements, we consider implications when one has prior knowledge of the span $\tspannorm{\vstar_\gamma}$.
In this case, one can set $H=\tspannorm{\vstar_\gamma}$ to obtain a regret bound of $\tbigo\big( \sqrt{\smash[b]{SA \varstar}} + \Gamma \tspannorm{\vstar_\gamma}SA \big)$. When the MDP is deterministic, the $\gamma$-regret is $T$-independent up to logarithmic factors, scaling as $\tbigo\big(\tspannorm{\vstar_\gamma}SA\big)$.
For stochastic MDPs, the leading term $\tbigo\big( \sqrt{\smash[b]{\tspannorm{\vstar_\gamma}SAT}}\big)$ matches the minimax lower bound for $\gamma$-regret.
Although this rate was previously achieved by \cite{he_nearly_2021} and \cite{ji_regret-optimal_2023}, their bounds depend on $\frac{1}{1-\gamma}$ in both the leading term and the lower-order terms. 
Now, $\tspannorm{\vstar_\gamma}$ can be as large as $\frac{1}{1-\gamma}$, but it has the potential to be bounded independently of $\gamma$, such as in weakly communicating MDPs where $\tspannorm{\vstar_\gamma} \leq 2\tspannorm{\hstar}$ \citep{wei_model-free_2020}.

We remark that while Theorem~\ref{thm:var_dependent_gamreg} is the the first explicit span-dependent bound for $\gamma$-regret, the analysis in \cite{hong_reinforcement_2025} can be slightly modified to yield a span-dependent bound of $\tbigo\big( \tspannorm{\vstar_\gamma}S \sqrt{AT} + \frac{S}{1-\gamma} \big)$ (see Appendix~\ref{sec:gamma_regret_ucbcvi_details} for details).
However, this result is not minimax optimal and still has a lower-order dependence on $\frac{1}{1-\gamma}$.
Their algorithm also requires prior knowledge of the span.
Theorem~\ref{thm:var_dependent_gamreg}, on the other hand, shows that we can attain a span-based bound without prior knowledge.
In particular, by setting $H = \frac{1}{1-\gamma}$, we attain a bound of $\tbigo\big( \sqrt{\smash[b]{\tspannorm{\vstar_\gamma}SAT}} + \frac{S^2A}{1-\gamma} \big)$.
We can even remove the lower-order dependence on $\frac{1}{1-\gamma}$ by setting $H=\sqrt{T/(S^3A)}$, similar to Corollary~\ref{cor:regret_no_prior} below.

\subsection{Main Results for Average-Reward Setting} \label{sec:average_main_results}

In this section we use our results from the discounted setting with a properly tuned $\gamma$ to approximate the average-reward setting.
That is, by setting the discount factor $\gamma$ to be large enough, the $\gamma$-regret bound achieved by FOCUS implies an optimal variance-dependent regret.
Our reduction hinges on the following lemma.
The proof follows in a straightforward manner from standard average-to-discounted reduction techniques \citep{wei_model-free_2020, zurek_span-agnostic_2025}, and we defer the complete proof to Section~\ref{sec:avg_to_disc_reduction_proof}.
\begin{lemma} \label{lem:avg_to_disc_reduction}
    Suppose the MDP is weakly communicating.
    For any $T\geq 1$ and $\gamma\in(0,1)$, it holds that 
    $\reg(T) \leq (1-\gamma)\bspannorm{\vstar_\gamma}T + \reg_\gamma(T).$
\end{lemma}
Lemma~\ref{lem:avg_to_disc_reduction} decomposes regret into $\gamma$-regret and an approximation error term $(1-\gamma)\tspannorm{\vstar_\gamma}T$.
By choosing $\gamma$ close to 1, this term becomes negligible, and the regret is bounded by $\gamma$-regret. 
This observation reinforces the notion that a span-based $\gamma$-regret bound is crucial for deriving optimal guarantees in the average-reward setting.
Indeed, when $\gamma$ is close to 1, applying Lemma~\ref{lem:avg_to_disc_reduction}  to a $\gamma$-regret bound that scales with $\frac{1}{1-\gamma}$ instead of the span would imply a vacuous regret bound.
We can now state our main theorem for the average-reward setting.
It follows by combining Theorem~\ref{thm:var_dependent_gamreg} and Lemma~\ref{lem:avg_to_disc_reduction}, taking $\gamma = 1 - \frac{1}{T}$, and using the fact that $\tspannorm{\vstar_\gamma} \leq 2\spannorm{\hstar}$ for all $\gamma\in(0,1)$.

\begin{theorem}[Variance-Dependent Regret Bound] \label{thm:var_dependent_reg}
	Suppose the MDP is weakly communicating.
    Let $T\geq 1$ and $\delta\in(0,1)$.
	For any $H \geq 2\tspannorm{\hstar}$, Algorithm~\ref{alg:oovi} with input $(T,\gamma = 1-\frac{1}{T},\delta,H)$ achieves, with probability at least $1-\delta$,
	\[
	    \reg(T) \leq \tbigo\Big( \sqrt{\mathrm{Var}^\star_{1-\frac{1}{T}} SA} + \Gamma HSA \Big).
	\]
\end{theorem}
Theorem~\ref{thm:var_dependent_reg} establishes the first minimax optimal variance-dependent regret bound for the average-reward setting.
Similar to our $\gamma$-regret result, the leading term depends on $\mathrm{Var}^\star_{1-\frac{1}{T}}$, so the regret adapts to the stochasticity of the environment.
We remark that \cite{talebi2018variance} provide a variance-dependent regret bound, but it only implies a suboptimal $\tbigo( DS\sqrt{AT})$ regret bound in the worst case.
Furthermore, their bound includes a lower order term of $\tbigo(D\sqrt{T})$, which means it cannot be optimal for deterministic MDPs.
Thus, as demonstrated by the following corollary, we have the first regret guarantee that is simultaneously optimal for stochastic and deterministic MDPs.

\begin{corollary}[Regret Bound with Prior Knowledge] \label{cor:regret_prior}
    Suppose the MDP is weakly communicating.
	Let $T\geq 1$ and $\delta\in(0,1)$.
    Algorithm~\ref{alg:oovi} with input $(T, \gamma=1-\frac{1}{T}, \delta,$ $H=2\spannorm{h^\star})$ satisfies
	\[
		\reg(T) \leq \tbigo\Big( \sqrt{\mathrm{Var}^\star_{1-\frac{1}{T}}SA} + \Gamma \spannorm{\hstar}SA \Big)
	\]
    with probability at least $1-\delta$.
    Consequently, Lemma~\ref{lem:var_star_bound} implies that with probability at least $1-2\delta$,
    \[
		\reg(T) \leq \tbigo\left( \sqrt{\spannorm{\hstar}SAT} + \spannorm{\hstar}S^2A \right).
	\]
    It follows that with probability at least $1-2\delta$,
    \[
        \reg(T) \leq \tbigo\left(\sqrt{\spannorm{\hstar}SAT}\right),
    \]
    provided that $T \geq \spannorm{\hstar}S^3 A$.
\end{corollary}

Corollary~\ref{cor:regret_prior} shows the optimal bounds that FOCUS attains across different regimes.
When the underlying MDP is deterministic, the regret scales as $\tbigo\big(\tspannorm{\hstar}SA\big)$, which is optimal and $T$-independent up to logarithmic factors.
For stochastic MDPs, the leading term is minimax optimal, while the lower-order term is significantly smaller than those incurred by existing algorithms.
We later show that this $\tspannorm{\hstar}S^2A$ lower order term is nearly optimal --- it could be improved at most by a factor of $S$ to $\tspannorm{\hstar}SA$.
Note that Corollary~\ref{cor:regret_prior} applies Theorem~\ref{thm:var_dependent_reg} with span bound $H=2\spannorm{h^\star}$, which requires prior knowledge of $\spannorm{h^\star}$.
We next consider the performance of FOCUS when we do not have prior knowledge of $\tspannorm{\hstar}$.

\begin{corollary}[Regret Bound without Prior Knowledge] \label{cor:regret_no_prior}
    Suppose the MDP is weakly communicating.
    Let $T\geq 1$ and $\delta\in(0,1)$.
    Algorithm~\ref{alg:oovi} with input $(T, \gamma=1-\frac{1}{T}, \delta, H = \sqrt{T/(S^3A)})$ satisfies, with probability at least $1-\delta$,
    \[
		\reg(T) \leq \tbigo\Big( \sqrt{\mathrm{Var}^\star_{1-\frac{1}{T}}SA} + \sqrt{SAT} \Big),
	\]
    provided that $T\geq\spannorm{\hstar}^2 S^3 A$.
    Consequently, Lemma~\ref{lem:var_star_bound} implies that with probability at least $1-2\delta$,
    \[
        \reg(T) \leq \tbigo\left(\sqrt{\left(\spannorm{\hstar}+1\right)SAT}\right)
    \]
    provided that $T \geq \spannorm{\hstar}^2 S^3 A$.
    It follows that with probability at least $1-2\delta$, we have 
    \[
		\reg(T) \leq \tbigo\Big( \sqrt{\left(\spannorm{\hstar}+1\right)SAT} + \spannorm{\hstar}^2 S^3 A \Big).
	\]
\end{corollary}

We show how the last two regret bounds follow from the first in Appendix~\ref{sec:regret_no_prior_proof}.
Corollary~\ref{cor:regret_no_prior} shows that our algorithm achieves significantly improved burn-in cost compared to that of previous work on priorless algorithms.
Indeed, PMEVI-DT attains minimax optimality for $T\geq \tspannorm{\hstar}^{10} S^{40} A^{20}$, whereas our algorithm attains minimax optimality for $T \geq \tspannorm{\hstar}^2 S^3 A.$
Furthermore, we prove a matching lower bound in Theorem~\ref{thm:no_prior_knowledge_H2_lb} showing that the lower-order dependence on the bias span cannot be improved beyond $\tspannorm{\hstar}^2$.
We remark that our choice of $H$ is optimized for the more interesting scenario that $\tspannorm{\hstar}\geq 1$ and leads to lower-order terms scaling quadratically in $\tspannorm{\hstar}$. If one cares about the possibility of $\spannorm{\hstar} \ll 1$, we could instead choose $H$ to be lower-order in $T$, such as $H=(T/(S^3A))^{1/(2+\varepsilon)}$, so that the leading term is $\tbigo(\sqrt{\tspannorm{\hstar}SAT})$ for large $T$ albeit with slightly worse lower-order terms.

Unlike the case with prior knowledge of $\tspannorm{\hstar}$, Theorem~\ref{thm:var_dependent_reg} does not immediately imply a simultaneously optimal bound for stochastic and deterministic MDPs. 
We claim it is still possible to obtain some form of $T$-independent bound for certain MDP instances.
Towards this end, consider running a diameter estimation procedure (i.e., \citealt{tarb_diameterest_2021}; \citealt{tuynman2024finding}) for at most $\sqrt{T/(S^3A)}$ steps.
With high probability, it will either terminate within $\text{poly}(DSA)$ steps and output $\widehat{D}$ satisfying $D \leq \widehat{D} \leq 2D$, or it will not terminate, in which case we set $\widehat{D}=\infty$.
For the remainder of the time steps, we run Algorithm~\ref{alg:oovi} with $H=\min\{\widehat{D}, \sqrt{T/(S^3A)}\}$, recovering the same bound as Corollary~\ref{cor:regret_no_prior} for stochastic MDPs 
and a bound of $\tbigo(\min\{\text{poly}(DSA), \sqrt{SAT} + \spannorm{\hstar}^2S^3A$\}) for deterministic MDPs, which is $T$-independent for strongly communicating deterministic MDPs, and still finite when $D=\infty$.

\subsection{Lower Bounds for Average-Reward Regret}
In this section we present lower bounds on the average-reward regret. Prior work in \citet{auer_near-optimal_2008} establishes a regret lower bound of $\Omega(\sqrt{\tspannorm{\hstar}SAT})$ when $T \geq \tspannorm{\hstar}SA$,\footnote{This lower bound was originally stated in terms of the diameter $D$, but for their hard instances $\tspannorm{\hstar}$ and $D $ differ by only a constant factor.} which has since been matched by several algorithms (\citealt{zhang_regret_2019, boone_achieving_2024}; our Corollaries \ref{cor:regret_prior} and \ref{cor:regret_no_prior}) when the horizon $T$ is sufficiently large.
In contrast to the large-$T$ regime, here we focus on lower bounds applicable for all $T$, in order to characterize the optimal burn-in cost of any algorithm.

We begin by formalizing the definition of algorithms to which our lower bounds apply. We define a \textit{horizon-$T$ algorithm} \texttt{Alg} to be a function from histories of length $\leq T$ to a distribution over actions, that is, a function $\bigcup_{0\leq t \leq T}\left(\St \times (\A \times \St)^t\right) \to \Delta(\A)$.
We note that as defined, such an algorithm only takes as input a sequence of elements of $\St$ and $\A$; intuitively speaking, any other data, such as the value of $T$ or prior knowledge of $\tspannorm{h^\star}$, must already be ``baked in'' to the algorithm. Hence by our definition, an ``algorithm'' (for horizon $T$) with prior knowledge of $\tspannorm{h^\star}$ is actually a family of horizon-$T$ algorithms, one for each value of $\tspannorm{h^\star}$.

Now we present the main theorem of this section, a lower bound on lower-order terms 
incurred by any algorithm that does not have prior knowledge of $\tspannorm{\hstar}$.
Note that this also implies a lower bound on the burn-in cost.
\begin{theorem}[Burn-In Lower Bound For Prior-Free Algorithms]
\label{thm:no_prior_knowledge_H2_lb}
    There is a universal constant $c\geq1$ such that the following holds.
    Let $S\geq 2$ and $A\geq 2$ be integers.
    Fix $\alpha \in [1,2)$ and a function $t \mapsto \beta_t$. Suppose that $T > SA(c \beta_T)^{\frac{4}{2-\alpha}}$ and $\beta_T \geq 1$. Then there exist two communicating MDPs $P_1$ and $P_2$, each with $S$ states and $A$ actions, such that no horizon-$T$ algorithm $\texttt{Alg}$ can satisfy,  for both $i=1$ and $i=2$, 
    \begin{align*}
        \E[\reg(T,P_i,\texttt{Alg})] \leq \sqrt{\beta_T \bspannorm{h^\star_{P_i}}SAT} + \beta_T SA\spannorm{h^\star_{P_i}}^\alpha.
    \end{align*}
\end{theorem}
We provide a proof sketch in Section~\ref{sec:no_prior_h2_lb_sketch} and a complete proof in Appendix~\ref{sec:no_prior_H2_lb_proof}.
Here we intend $\beta_T$ to be used to encapsulate $\widetilde{O}(1)$ terms; Theorem \ref{thm:no_prior_knowledge_H2_lb} states that for any $\alpha < 2$, we can find $T$ such that no single horizon-$T$ algorithm can enjoy regret bounds of the form $\widetilde{O}(\sqrt{\tspannorm{\hstar}SAT} + \tspannorm{\hstar}^\alpha SA)$ simultaneously for two certain MDPs $P_1$ and $P_2$.
The two MDPs have $\tspannorm{\hstar_{P_1}} \gg \tspannorm{\hstar_{P_2}}$. With prior knowledge, formally a different horizon-$T$ algorithm would be applied to each MDP $P_1, P_2$, and furthermore, we would not expect a horizon-$T$ algorithm designed for MDPs with bias span $\leq \tspannorm{\hstar_{P_2}}$ to enjoy a nonvacuous regret bound when deployed on $P_1$. So in short, Theorem \ref{thm:no_prior_knowledge_H2_lb} is not a counterexample to the type of theorem that one proves when designing algorithms that use prior knowledge, and in particular does not contradict our Corollary \ref{cor:regret_prior}. However, this lower bound does prohibit any minimax-optimal (for large $T$) algorithm without prior knowledge from obtaining a better $\tspannorm{\hstar}$ dependence in its lower-order terms than $\tspannorm{\hstar}^2$. This is matched by our prior-knowledge-free Corollary \ref{cor:regret_no_prior}.

Theorem \ref{thm:no_prior_knowledge_H2_lb} demonstrates a ``price of adaptivity'' for the burn-in cost, that is, a gap between what is achievable with and without prior knowledge. In particular, this gap is established by combining the above lower bound and the strictly smaller regret upper bound provided by our prior-knowledge-based Corollary \ref{cor:regret_prior}, which achieves $\widetilde{O}(\sqrt{\tspannorm{\hstar} T} + \tspannorm{\hstar})$ when applied to instances with $S,A \leq O(1)$. Note that previous results are insufficient for establishing this gap, as the only other algorithm which uses prior knowledge and achieves minimax-optimal regret for large $T$ has a burn-in cost scaling with $\tspannorm{\hstar}^6$ even when $S, A \leq O(1)$ \citep{zhang_regret_2019}. The only result of a similar nature for the average-reward regret of which we are aware is \citet[Lemma 3]{fruit_near_2019}, which shows an exponential lower bound on the burn-in cost but only for algorithms achieving a logarithmic regret.
One particularly interesting feature of the gap implied by Theorem \ref{thm:no_prior_knowledge_H2_lb} is that this contrasts the simulator (a.k.a.\ generative model) setting, where recent work has characterized the sample complexity in terms of $\tspannorm{\hstar}$ and demonstrated no gap between algorithms which do and do not possess prior knowledge \citep{wang_near_2022, zurek_span-based_2025, zurek_span-agnostic_2025}.

Next, we show a burn-in cost lower bound applicable even to algorithms with prior knowledge.
\begin{theorem}[General Burn-In Lower Bound] \label{thm:burn_in_lb}
	Let $S\geq 2$ and $A \geq 2$ be integers, and let $D \geq 4\left\lceil \log_A S \right\rceil$. 
	For any horizon-$T$ algorithm \texttt{Alg}, there exists an MDP $P$ with $S$ states, $A$ actions, and diameter at most $D$ such that for all $T \leq \frac{1}{32}DSA$,
    $\E[\reg(T,P,\texttt{Alg})] \geq \frac{T}{4}.$
\end{theorem}
We defer the proof, which uses standard constructions, to Appendix~\ref{sec:burn_in_lb_proof}. We emphasize that this result applies to any horizon-$T$ algorithm. 
Since $D \geq \tspannorm{h^\star}$ (in fact they are equal up to a constant factor for this MDP), Theorem \ref{thm:burn_in_lb} implies that any algorithm with a sublinear-in-$T$ regret bound  must have a burn-in requirement of $T \geq \Omega(\tspannorm{h^\star}SA)$. In particular, combining with the lower bound from \citet{auer_near-optimal_2008}, we see that no algorithm can have regret below $\Omega(\sqrt{\tspannorm{\hstar}SAT} + \tspannorm{\hstar}SA)$, matching our Corollary \ref{cor:regret_prior} up to an additional $S$ in the additive burn-in term and $\widetilde{O}(1)$ factors. We conjecture that this lower bound is nearly tight up to a multiple of $\tbigo(1)$ and that the factor of $S$ in our upper bound could be removed, although we believe this may be very challenging and the techniques used to do so in the inhomogeneous episodic setting \citep{zhang_settling_2024} would not apply in the infinite horizon case.

\section{Technical Highlights} \label{sec:technical_highlights}

In this section, we discuss our algorithmic and analytical contributions in the context of related work.
We also include a proof sketch for Theorem~\ref{thm:no_prior_knowledge_H2_lb}.

\subsection{Algorithmic Improvements Over Prior UCB-based Approaches}

Our algorithm FOCUS builds on prior model-based algorithms for the discounted setting, particularly UCBVI-$\gamma$ \citep{he_nearly_2021} and $\gamma$-UCB-CVI \citep{hong_reinforcement_2025},
but introduces several crucial modifications that enable variance-dependent bounds and eliminate extraneous dependence on $\frac{1}{1-\gamma}$.
One main novelty is how the optimistic Q-estimate is updated.
Previous algorithms initialize the estimate using $\widehat{Q}_1\gets \frac{1}{1-\gamma}\mathbf{1}$, and then at each time step $t\in[T]$ perform the following one-step value iteration:
\begin{equation*}
    \widehat{Q}_{t+1}(s,a) \gets r(s,a) + \gamma P^t_{s,a}\widehat{V}_t + \gamma b_t(s,a).
\end{equation*}
Here, $\widehat{V}_t = M\widehat{Q}_t$ (in UCBVI-$\gamma$) or
$\widehat{V}_t = \clipH\big(M\widehat{Q}_t\big)$ (in $\gamma$-UCB-CVI).
We first discuss the bonus term $b_t(s,a)$.
UCBVI-$\gamma$ uses a Bernstein-style bonus similar to that of the UCBVI algorithm for the episodic setting \citep{azar_minimax_2017}.
UCBVI obtains a regret bound which is optimal in the leading term but suboptimal in lower-order terms.
This suboptimality is due to an extra term in the bonus; indeed, in the episodic setting,  \cite{zhang_is_2021} shows a similar term to be unnecessary for optimism.
Hence, UCBVI-$\gamma$ incurs suboptimal lower-order terms for the same reason.
The MVP algorithm \citep{zhang_is_2021} removes this extra term from the bonus to achieve significantly improved lower order terms in the episodic setting.

Secondly, as mentioned above, $\gamma$-UCB-CVI uses a clipping step to ensure $\tspannorm{\widehat{V}_t} \leq H$.
Their analysis has steps that involve upper bounding $\tspannorm{\widehat{V}_t}$, and  clipping allows one to replace some factors of $\frac{1}{1-\gamma}$ with $H$.
However, $\gamma$-UCB-CVI uses a Hoeffding-style bonus, resulting in a suboptimal leading term.
Still, the result of clipping is that the leading term of their $\gamma$-regret bound depends on $H$ instead of $\frac{1}{1-\gamma}$.

In light of the discussion above, an immediate idea is to combine clipping with the sharpest Bernstein-style bonus similar to that of the MVP algorithm.
This strategy does yield an improvement, and we believe it would achieve a $\gamma$-regret of $\tbigo\big( \sqrt{\tspannorm{\vstar_\gamma}SAT} + HS^2A + \frac{SA}{1-\gamma} \big)$. 
While this bound is better than the state-of-the-art for $\gamma$-regret, the lower-order factor of $\frac{1}{1-\gamma}$ would prevent us from obtaining a variance-dependent or priorless span-based regret bound for the average-reward setting. 
Indeed, applying our average-to-discounted reduction (Lemma~\ref{lem:avg_to_disc_reduction}) to derive an optimal regret bound would require setting $\gamma=1-\sqrt{HT/(SA)}$, and the resulting bound would be $\tbigo\big( \sqrt{HSAT} + HS^2A\big)$.
This bound is not variance-dependent, and moreover the leading term depends on the tuning parameter $H$ instead of $\tspannorm{\hstar}$, which means the algorithm would only be optimal with prior knowledge of $\tspannorm{\hstar}$.

An intuitive explanation for the suboptimality of these prior methods is that while their estimated $\widehat{Q}_t$ eventually converges to $\qstar$ when $t\to\infty$, for finite $t$ the estimation error is significant when $\gamma$ is close to $1$, since they perform only one step of value iteration at a time.
To be more precise, let $\widehat{\mathcal{T}}_t$ denote the empirical Bellman operator used at time $t$.
The estimate $\widehat{Q}_{t+1}$ may remain significantly larger than the fixed point of $\widehat{\mathcal{T}}_t$, especially for small $t$.
This results in a detrimental dependency on $\frac{1}{1-\gamma}$ in lower order terms even if we use clipping.
In particular, as $\gamma$ gets closer to 1, $\widehat{Q}_1$ is initialized with larger entries, and it takes longer for $\widehat{Q}_t$ to converge to the best values supported by the data.
In short, value iteration takes on the order of $\frac{1}{1-\gamma}$ steps to approximately converge, while the statistical error converges to $0$ at an unrelated and potentially much faster rate, especially for low-variance MDPs or average-reward settings where $\gamma$ is tuned to be very large.

To address this issue, our algorithm FOCUS fully optimizes the Q-estimate at the beginning of each episode $k$.
Concretely, the algorithm iteratively applies the empirical Bellman operator $\widehat{\mathcal{T}}_k$ until convergence, producing an estimate $\widehat{Q}_k$ that fully exploits all data collected to that point.
This mechanism of fully optimizing is also a feature of algorithms that utilize the EVI subroutine, suggesting that full exploitation of available data is crucial for achieving optimal span-based bounds in the average-reward setting.
With this strategy, we obtain a $\gamma$-regret bound without dependence on $\frac{1}{1-\gamma}$, which allows us to solve the average-reward problem with a simple, UCB-based approach.

\subsection{How Full Optimization Helps in Regret Analysis}

A key technical challenge that illustrates the necessity of fully optimizing lies in controlling $\tind$, a term in our regret decomposition which accounts for changes in the value estimate along the learner's trajectory.
To see this, we first note that the prior work in \cite{hong_reinforcement_2025} bounds an analogous quantity in the analysis of $\gamma$-UCB-CVI according to
\begin{align*}
    \sum_{t=1}^T \left(\widehat{V}_{t-1}(s_{t+1}) - \widehat{V}_{t}(s_t) \right) 
    &\leq \sum_{t=1}^T \left(\widehat{V}_{t-1}(s_{t+1}) - \widehat{V}_{t+1}(s_{t+1}) \right) + \frac{1}{1-\gamma} \\
    & \leq \sum_{s\in\St}\sum_{t=1}^T \left(\widehat{V}_{t-1}(s) - \widehat{V}_{t+1}(s)\right) + \frac{1}{1-\gamma} \\
    &\leq O\left( \frac{S}{1-\gamma} \right). 
\end{align*}
Here, the second inequality holds because the value estimates are monotonically decreasing with $t$, a property they enforce by taking minimums.
These calculations show that the scale of this term depends on the possible range of value estimates, which, under one-step value iteration updates, can be as large as $\frac{1}{1-\gamma}$.
In contrast, by fully running the value iteration procedure, along with the use of clipping, our value estimates lie in a range of $H$ times the number of episodes, which is only logarithmic in $T$ due to the doubling trick.
Formally, letting $m$ be the total number of episodes and $t_k$ be the time at the start of the $k$th episode, a telescoping argument yields
\[
\tind 
= \sum_{k=1}^m \sum_{t=t_k}^{t_{k+1}-1} \left(\widehat{V}_k(s_{t_{k+1}}) - \widehat{V}_k(s_{t_k} \right)
\leq \sum_{k=1}^m \left(\widehat{V}_k(s_{t_{k+1}}) - \widehat{V}_k(s_{t_k} \right)
\leq mH
\leq \tbigo(HSA).
\]

\subsection{Comparison to EVI-Based Approaches}

The two existing minimax-optimal algorithms for the average-reward setting, EBF \citep{zhang_regret_2019} and PMEVI-DT \citep{boone_achieving_2024}, are both   based on EVI.
The way that these algorithms refine previous EVI-based approaches suggest that obtaining an optimal span-dependent regret requires exploiting structural information encoded in the optimal bias function $\hstar$.
We elaborate on this point by examining the state-of-the-art PMEVI-DT algorithm, whose name reflects two central modifications to standard EVI.
We remark that EBF is an earlier attempt in this direction, in that it estimates bias differences to shrink the confidence set, but incorporates this restriction through a step that is not efficiently computable.

At each episode, PMEVI-DT runs a ``BiasEstimation'' subroutine to construct a confidence set for $\hstar$.
The extended Bellman operator used in EVI is then combined with a ``projection'' step (the P in PMEVI-DT) that constrains the possible models to those with optimal bias in the confidence set.
Their extended Bellman operator also incorporates a ``mitigation'' step (the M in PMEVI-DT) that uses the bias confidence region to tighten a Bernstein-type variance constraint.
Despite this machinery, PMEVI-DT still requires either prior knowledge of $\spannorm{\hstar}$ or a condition like $T\geq \spannorm{\hstar}^5$ to achieve minimax-optimal regret.

Our results show that under the same type of assumptions, our UCB-based algorithm sufficiently exploits $\hstar$ without explicitly estimating it.
The span-clipping component of the empirical Bellman update replaces the projection and mitigation steps of PMEVI-DT with a simple, easily interpretable operation.
Specifically, span-clipping prevents the value estimate from being overly optimistic;
a smaller clipping threshold $H$ reduces exploration and increases immediate exploitation of known high-reward actions. 
This span-clipping technique was introduced by \cite{fruit_efficient_2018} as part of the EVI-based algorithm SCAL.
While SCAL obtains a span-based regret bound, it does not maintain sharp confidence regions and consequently its regret suffers from extra factors of $\tspannorm{\hstar}$ and $S$.
The more involved bias estimation techniques of EBF and PMEVI-DT circumvent these issues and produce sharp confidence regions with bounded bias spans, but our algorithm achieves optimal regret bounds with the simpler combination of span-clipping and a sharp Bernstein-style bonus.

Finally, given the success of UCB-based approaches in the episodic setting, we remark that it is at least somewhat surprising that such algorithms have yet to be thoroughly studied in the infinite-horizon average-reward setting. 
We suggest two contributing factors.
First, our approach utilizes an average-to-discounted reduction, a strategy which only recently has been shown to yield optimal span-based bounds \citep{zurek_span-based_2025, zurek_span-agnostic_2025}. 
In these works which derive an optimal span-based bound via an average-to-discounted reduction, a crucial step is tight analysis of variance-dependent quantities to remove factors of $\frac{1}{1-\gamma}$. 
The analogous step in our results is Lemma~\ref{lem:var_star_bound}, which was vital in successfully applying the reduction.
Secondly, since the seminal work of \cite{auer_near-optimal_2008}, the most well-studied and successful algorithms for the online infinite-horizon average-reward setting  have been EVI-based, so the most natural route for obtaining a minimax-optimal algorithm was to refine these existing works.

\subsection{Proof Sketch for Theorem~\ref{thm:no_prior_knowledge_H2_lb}}
\label{sec:no_prior_h2_lb_sketch}

We sketch the construction underlying Theorem~\ref{thm:no_prior_knowledge_H2_lb}, which shows that without prior knowledge of the bias span, any algorithm must incur a burn-in cost of order $\spannorm{\hstar}^2SA$.
Consider the MDPs $P_1$ and $P_2$ in Figure~\ref{fig:lower_bound_sketch}.
The MDPs, which both have two states and two actions, are nearly identical. For both we let state $1$ be the initial state.
In state 1 of both MDPs, the $\stay$ action yields reward 1/2 and remains in state 1, and the $\leave$ action yields reward 0 and has a small probability ($1/B$, for a parameter $B>2$) of transiting to state 2.
Furthermore, in both MDPs the $\leave$ action in state 2 yields reward 0 and transits to state 1.
The difference between $P_1$ and $P_2$ is the $\stay$ action in state 2.
In both MDPs it yields maximum reward 1, but in $P_1$ it remains in state 2, while in $P_2$ it transits back to state 1.
It follows that in $P_1$, state 2 is an absorbing high-reward region; the optimal policy is to reach state 2 and stay there for an average reward of 1.
In $P_2$, on the other hand, state 2 offers no long-term benefit, and the optimal policy is to stay in state 1 for an average reward of 1/2.
Additionally, a direct calculation confirms that the span of the optimal bias function in $P_1$ is $B$ (reflecting the long delay before being able to collect reward 1), while that of $P_2$ is 1/2.

\begin{figure}[ht]
    \centering
    \begin{minipage}{0.45\textwidth}
        \centering
        \begin{tikzpicture}[ -> , >=stealth, shorten >=2pt , thick, node distance =2cm, scale=0.9]

            \node [circle, draw] (one) at (-3 , 0) {1};
            \node [circle, draw] (two) at (0 , 0) {2};
            \node [circle, draw, fill, inner sep=0.03cm] (dot1) at (-3 , -1) {};
            
            \path (one) edge [loop left, looseness=15] node [left] {$a=\stay, r=\frac{1}{2}$}  (one) ;
            \path (two) edge [loop right, looseness=15] node [pos=1.05, above right] {$a=\stay, r=1$}  (two) ;
            
            \draw[-] (one) to[bend right] node[left] {$a=\leave, r=0$} (dot1);
            \draw[->, dashed] (dot1) to[bend right] node[right] {$1-\frac{1}{B}$} (one);
            \draw[->, dashed] (dot1) to[bend right] node[below] {$\frac{1}{B}$} (two);
        
            \draw[->] (two) to[bend right] node[above] {$a=\leave, r=0$} (one);
    
        \end{tikzpicture}
        \\[1ex] 
        $P_1$
    \end{minipage}
    \hfill 
    \begin{minipage}{0.45\textwidth}
        \centering
        \begin{tikzpicture}[ -> , >=stealth, shorten >=2pt , thick, node distance =2cm, scale=0.9]

            \node [circle, draw] (one) at (-3 , 0) {1};
            \node [circle, draw] (two) at (0 , 0) {2};
            \node [circle, draw, fill, inner sep=0.03cm] (dot1) at (-3 , -1) {};
            
            \path (one) edge [loop left, looseness=15] node [left] {$a=\stay, r=\frac{1}{2}$}  (one) ;
            
            \draw[-] (one) to[bend right] node[left] {$a=\leave, r=0$} (dot1);
            \draw[->, dashed] (dot1) to[bend right] node[right] {$1-\frac{1}{B}$} (one);
            \draw[->, dashed] (dot1) to[bend right] node[below] {$\frac{1}{B}$} (two);
        
            \draw[->] (two) to[bend right] node[pos=0.55, above] {$a=\leave, r=0$} (one);

            \draw[->] (two) to[bend right=115, looseness=1.3] node[midway, above]{$a=\stay, r=1$} (one);
    
        \end{tikzpicture}
        $P_2$
    \end{minipage}
    \caption{An example of the MDPs used in the proof of Theorem~\ref{thm:no_prior_knowledge_H2_lb}. Here each state-action pair is annotated with its reward. If the transition associated with a state-action pair is deterministic, it is denoted with a solid arrow. If it is stochastic, it is represented as a solid line splitting into multiple dashed arrows to different states, each annotated with the associated probability of that transition. The MDPs are parameterized by $B>2$, both have starting state $1$, and differ only in the transition distribution of the \texttt{stay} action of state $2$. In $P_1$ an optimal stationary policy traverses to state $2$ and stays there, while in $P_2$ an optimal stationary policy remains in state $1$.}
    \label{fig:lower_bound_sketch}
\end{figure}
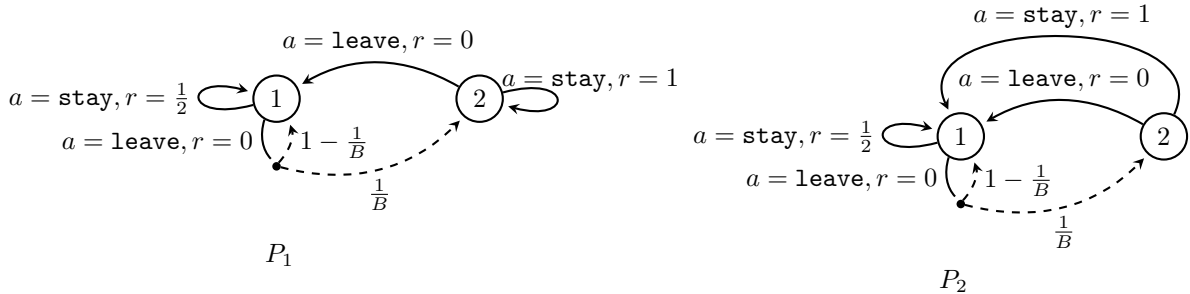

Now, suppose a learning algorithm has no structural information about the underlying MDP and promises a sublinear regret bound on $P_1$.
To achieve sublinear regret on $P_1$, the learner clearly must reach state 2.
Otherwise, the learner would incur at least regret 1/2 per time step.
Moreover, reaching state 2 requires taking the $\texttt{leave}$ action $B$ times in expectation.
On $P_2$, however, these exploration attempts are wasteful since $\texttt{leave}$ yields reward 0 with no potential to recoup reward in state 2.
Yet, since $P_1$ and $P_2$ only differ on the $\texttt{stay}$ action in state $2$, even when the true MDP is $P_2$, the algorithm must still reach state $2$ and collect data there to guard against the possibility that the MDP is actually $P_1$. Note that this deduction is only valid without prior knowledge, because with knowledge that the bias span is at most $1/2$, the possibility of $P_1$ could be eliminated without reaching state $2$ since $\tspannorm{h^\star_{P_2}} >1/2$.
Thus, any prior-knowledge-free algorithm which achieves less than $T/2$ regret on $P_1$ must incur at least $\Omega(B)$ regret on $P_2$.
We summarize the preceding argument in the following intermediate result.
\begin{lemma}[Simplified Version of Theorem~\ref{thm:lower_bound_prior}] \label{lem:informal_lb}
    There is a universal constant $c\in(0,1)$ so that the following holds.
    Fix $T,B$, and let $\texttt{Alg}$ be any horizon-$T$ algorithm. There exists two MDPs $P_1$ and $P_2$ such that $\spannorm{\hstar_{P_1}}=B$, $\spannorm{\hstar_{P_2}}=1/2$, and
    \[ \E_{P_1}[\reg(T)] < T/4 \implies \E_{P_2}[\reg(T)] \geq cB. \]
\end{lemma}
We now sketch how Theorem~\ref{thm:no_prior_knowledge_H2_lb} follows from Lemma~\ref{lem:informal_lb}, noting that in this discussion we will ignore factors of $S$ and $A$ for ease of presentation.
Let $\alpha\in[1,2)$ and $\beta\geq 1$ be arbitrary. For a sufficiently large $T$ --- specifically $T$ satisfying $\frac{\beta^2}{c^2}(T^{3/4} + T^{\alpha/2}) < T/4$ and $\sqrt{T/2} + 1/2 <\sqrt{T}$ --- suppose towards a contradiction that there exists a horizon-$T$ algorithm that obtains a $\beta\big(\sqrt{\tspannorm{\hstar}T} + \tspannorm{\hstar}^\alpha\big)$ regret bound in expectation for any MDP with $SA=O(1)$.
We then choose $B=\frac{\beta}{c}\sqrt{T}$ and let MDPs $P_1$ and $P_2$ be as in Lemma~\ref{lem:informal_lb}.
Next, we compute that on $P_1$ the algorithm obtains regret
\[
\E_{P_1}[\reg(T)]
\leq
\beta(\sqrt{BT} + B^\alpha) 
\leq 
\beta \left( \sqrt{T\frac{\beta}{c}\sqrt{T}} + \left(\frac{\beta}{c}\sqrt{T}\right)^\alpha \right)
\leq
\frac{\beta^2}{c^2}\left(T^{3/4} + T^{\alpha/2}\right)
<
T/4,
\]
while on $P_2$ the algorithm obtains regret
\[
\E_{P_2}[\reg(T)]
\leq
\beta\left(\sqrt{T/2} + \left(1/2\right)^\alpha\right) < \beta\sqrt{T} = cB
\]
which contradicts Lemma~\ref{lem:informal_lb}.

We emphasize that the role of no prior knowledge is implicit but crucial.
The theorem fixes a single horizon-$T$ algorithm and evaluates it on two MDPs with very different optimal bias spans.
If prior knowledge were available, a different horizon-$T$ algorithm could be deployed on each instance, and the above contradiction would not arise.
Prior work (e.g., \cite{fruit_efficient_2018}) demonstrates how prior knowledge allows the learner to aggressively utilize span-clipping and exploit earlier, thereby avoiding unnecessary exploration.
Our hard MDP instances illustrate that, in general, such exploitation is impossible without prior knowledge.
Indeed, without prior knowledge the algorithm must explore significantly longer to perform optimally in instances with large bias span, but this exploration results in worse burn-in cost in instances with small bias span.

\section{Conclusion}
We developed the first algorithm for both average-reward regret and $\gamma$-regret that is simultaneously minimax-optimal and variance-dependent. Additionally, our average-reward regret bounds have optimal lower-order dependence on $\tspannorm{\hstar}$, and we proved lower bounds which revealed a fundamental gap in what is achievable with and without prior knowledge. 
One open problem is to eliminate the $\Gamma$ factor from the lower-order terms of Theorems \ref{thm:var_dependent_gamreg} and \ref{thm:var_dependent_reg}, which has recently been done in the inhomogeneous episodic setting \citep{zhang_settling_2024} but appears more challenging in infinite-horizon settings.

\section*{Acknowledgment}

G.\ Zamir, M.\ Zurek, and Y.\ Chen acknowledge support by National Science Foundation grants CCF-2233152 and DMS-2023239, by a Cisco Systems Fellowship, and by a Vilas Associates Award.

\bibliographystyle{plainnat}
\bibliography{refs}

\appendix

\section{More Related Work}
\label{sec:related_work}
Here we discuss additional related work.

\paragraph{Average-Reward Simulator and Offline Settings}
A complementary problem setting to online RL is the offline/simulator setting, where the goal is to learn a $\varepsilon$-optimal policy $\pi$ (such that $\rho^\pi(s) \geq \rho^\star(s) - \varepsilon$ for all $s \in \St$) from a fixed/simulator-generated dataset with the minimum number of samples. A sequence of works obtained sharper sample complexity bounds with relaxed assumptions on the environmental structure (e.g. \citealt{jin_efficiently_2020, jin_towards_2021, wang_near_2022, wang_optimal_2024, li_stochastic_2024}), culminating with the optimal $\spannorm{\hstar}$-based sample complexity bound of $\tbigo(SA\frac{\spannorm{\hstar}}{\varepsilon^2})$ shown by \citet{zurek_span-based_2025}, matching a lower bound due to \citet{wang_near_2022}. However, this result required prior knowledge of $\spannorm{\hstar}$, leaving the question open of whether the optimal sample complexity could be obtained by algorithms without prior knowledge. After extensive research effort \citep{neu_dealing_2024, tuynman2024finding, zurek_plug-approach_2024, lee_near-optimal_2025}, this question was answered affirmatively by \citet{zurek_span-agnostic_2025}. \citet{tuynman2024finding} and \citet{zurek_span-based_2025} show various hardness results related to estimating $\spannorm{\hstar}$ and $\spannorm{\hstar}$-based PAC guarantees with online environment access. A very common approach throughout these works is to reduce the average-reward problem to a discounted one; obtaining sharp variance bounds for the discounted problem, somewhat analogous to our Lemma \ref{lem:var_star_bound}, plays a key role in all minimax-optimal approaches \citep{zurek_span-based_2025, zurek_plug-approach_2024, zurek_span-agnostic_2025}. All of the aforementioned work is for the generative model setting, where a dataset with uniform coverage can be sampled. Recent works have also studied offline settings with more general data sampling patterns \citep{gabbianelli_offline_2023, ozdaglar_offline_2024, zurek2025optimal}.

\paragraph{Episodic Online RL}
Studying the theoretical limits of regret for episodic online RL has been one of the most fundamental problems in RL theory. Hence, we cannot provide a comprehensive review of work on this topic, but we discuss a few related works with strong connections to our own. 
\cite{azar_minimax_2017} establishes minimax regret bounds in the episodic setting with the UCBVI algorithm, which uses a Bernstein-style bonus. \cite{zanette_tighter_2019} developed the EULER algorithm which obtains both minimax-optimal and variance-dependent regret bounds.
\cite{zhang_is_2021} greatly improves lower order terms by using an even sharper Bernstein-style bonus in their MVP algorithm.
Later refinements to the MVP algorithm in \cite{zhou_sharp_2023} and \cite{zhang_settling_2024} yield optimal lower order terms and variance-dependent regret bounds that adapt to the difficulty of the environment.

\section{\texorpdfstring{$\gamma$}{Gamma}-Regret Bound of \texorpdfstring{$\gamma$}{Gamma}-UCB-CVI} \label{sec:gamma_regret_ucbcvi_details}

In this section we show how the analysis of $\gamma$-UCB-CVI in \cite{hong_reinforcement_2025} can be modified slightly to obtain a span-based $\gamma$-regret bound.
The quantities $P^t, V_t$ and $N_t$ refer to the empirical transition kernel, value estimate, and empirical state-action counts, respectively, of $\gamma$-UCB-CVI at time $t$.
$\beta = \widetilde{\Theta}(\spannorm{\vstar_\gamma}\sqrt{S})$ is a constant in the bonus term which is large enough to ensure optimism.
Additionally, in their analysis they often immediately bound $\spannorm{\vstar_\gamma} \leq 2\spannorm{\hstar}$ and end up with factors of $\spannorm{\hstar}$ in intermediate steps.
Below, we leave in the factors of $\spannorm{\vstar_\gamma}$.

Via a concentration argument on $|(P^{t-1}_{s_t,a_t} - P_{s_t,a_t})V_{t-1}|$ (their Lemma 3), \citeauthor{hong_reinforcement_2025} obtains that with high probability
\[
r(s_t,a_t) \geq V_t(s_t) - \gamma P_{s_t,a_t}V_{t-1} - \frac{2\beta}{\sqrt{N_{t-1}(s_t,a_t)}}.
\]
Subsequently, \citeauthor{hong_reinforcement_2025} decompose the average-reward regret as
\begin{align*}
    \reg(T)
    &= \sum_{t=1}^T (\rhostar - r(s_t,a_t)) \\
    &\leq \sum_{t=1}^T \left( \rhostar - V_t(s_t) + \gamma P_{s_t,a_t}V_{t-1} + \frac{2\beta}{\sqrt{N_{t-1}(s_t,a_t)}}\right) \\
    &= \underbrace{\sum_{t=1}^T (\rhostar - (1-\gamma)V_t(s_t))}_{(a)} \quad + \quad \gamma \underbrace{\sum_{t=1}^T (V_{t-1}(s_{t+1} - V_t(s_t)))}_{(b)} \\
    &\qquad\qquad + \gamma \underbrace{\sum_{t=1}^T (P_{s_t,a_t}V_{t-1} - V_{t-1}(s_{t+1}))}_{(c)} \quad + \quad 2\underbrace{\beta \sum_{t=1}^T \frac{1}{\sqrt{N_{t-1}(s_t,a_t}}}_{(d)}.
\end{align*}
To instead analyze $\reg_\gamma(T) = \sum_t ((1-\gamma)\vstar_\gamma(s_t) - r(s_t,a_t)$, we simply replace $\rhostar$ with $(1-\gamma)\vstar_\gamma(s_t)$.
With this replacement, term $(a)$ vanishes due to optimism (their Lemma 4).
For the other terms, the bounds in the original proof still hold.
In particular, term $(b)$ is bounded by $O(\frac{S}{1-\gamma})$, term $(c)$ is bounded by $\tbigo(\spannorm{\vstar_\gamma}\sqrt{T})$, and term $(d)$ is bounded by $O(\beta\sqrt{SAT})\leq \tbigo(\spannorm{\vstar_\gamma}S\sqrt{AT})$.
Recombining terms, we obtain that with high probability,
\[
\reg_\gamma(T) \leq \tbigo\left( \spannorm{\vstar_\gamma}S\sqrt{AT} + \frac{S}{1-\gamma} \right).
\]

\section{Technical Lemmas}

\begin{lemma}[Bernstein's Inequality, Theorem 3 in \citealt{maurer_empirical_2009}] \label{lem:bernstein}
Let $Z, Z_1, \dots, Z_n$ be i.i.d.\ random variables with values in $\left[ c_{\min}, c_{\max} \right]$ for some constants $c_{\min} < c_{\max}$.
Set $c = c_{\max} - c_{\min}$. 
Let $\delta > 0$.
Then we have with probability at least  $1-\delta$ that
\[
	\left| \mathbb{E}[Z] - \frac{1}{n}\sum_{i=1}^{n} Z_i \right|
	\leq 
	\sqrt{\frac{2\mathrm{Var}(Z) \log\left( 2 /\delta\right)}{n}} + \frac{c \log\left( 2 /\delta \right) }{3n} 
.\] 
\end{lemma}

\begin{lemma}[Empirical Bernstein's Inequality, Theorem 4 in \citealt{maurer_empirical_2009}] \label{lem:empirical_bernstein}
Let $Z,$ $ Z_1, \dots, Z_n$ be i.i.d.\ random variables with values in $\left[ c_{\min}, c_{\max} \right]$ for some constants $c_{\min} < c_{\max}$.
Set $c = c_{\max} - c_{\min}$. 
Denote $\overline{Z} = \frac{1}{n}\sum_{i=1}^{n} Z_i$ and $\widehat{\mathrm{Var}}_n = \frac{1}{n} \sum_{i=1}^{n} \left( Z_i - \overline{Z} \right)^2$.
Let $\delta > 0$.
Then we have with probability at least  $1-\delta$ that
\[
	\left| \mathbb{E}[Z] - \frac{1}{n}\sum_{i=1}^{n} Z_i \right| 
	\leq 
	\sqrt{\frac{2\widehat{\mathrm{Var}}_n \log\left( 2 / \delta \right) }{n-1}} + \frac{7 c \log \left( 2 / \delta \right) }{3(n-1)}
.\] 
\end{lemma}

\begin{lemma}[Lemma 22 in \citealt{zhou_sharp_2023}] \label{lem:monotonicity}
For any two nonnegative constants $c_1,c_2$ satisfying $2c_1^2 \leq c_2$,
let $f : \Delta([S]) \times \left[ 0, 2C \right]^S \times \R \times \R \to \R$ be defined by
\[
	f(p, v, n, u) = pv + \max \left\{ c_1 \sqrt{\frac{\V(p,v) u}{n}}, c_2 \frac{Cu}{n} \right\} 
.\] 
Then for all $p\in \Delta([S]), v\in \left[ 0, 2C\right]^S,$ and $n,u>0$, $f$ is non-decreasing in $v$, i.e.
\[
f(p,v,n,u) \geq f(p,v',n,u) \quad \forall v,v'\in \left[ 0, 2C \right]^S \text{ satisfying } v\geq v' 
.\] 
\end{lemma}

\begin{lemma}[Lemma 19 in \citealt{zhou_sharp_2023}] \label{lem:varsq}
Let $X$ be a random variable with $\sup |X| \leq c$ for some constant $c\geq 0$.
Then 
\[
\var\left( X^2 \right) \leq 4c^2 \var(X)
.\] 
\end{lemma}

\begin{lemma} \label{lem:ab}
If $x \leq a\sqrt{x} + b$ for $a,b>0$, then  $x \leq 2a^2 + 2b$.
\end{lemma}

\begin{lemma}[Bernoulli's Inequality] \label{lem:bernoulli}
Let $r\geq 1$ and $x\geq -1$.
Then $(1+x)^r \geq 1+rx$.
\end{lemma}

\begin{lemma}[Rearrangement Inequality] \label{lem:rearrangement}
Let $x_1\leq\dots\leq x_n$ and $y_1\leq\dots\leq y_n$ be real numbers.
For every permutation $\sigma$ of $1,\dots,n$, we have
\[
x_1y_n + \dots + x_ny_1
\leq
x_1y_{\sigma(1)} + \dots + x_n y_{\sigma(n)}
\leq
x_1y_1 + \dots + x_n y_n.
\]
\end{lemma}

\section{Proof of Theorem~\ref{thm:var_dependent_gamreg}} \label{sec:var_dependent_gamreg_proof}

We provide an overview of notation used in the proof.
We let $\ind$ denote the indicator function, meaning that for an event $\mathcal{E}$, we have $\ind(\mathcal{E})=1$ if $\mathcal{E}$ holds and $\ind(\mathcal{E})=0$ otherwise.
We denote by $m$ the number of episodes.
For each $k\in[m]$, we write $t_k$ to denote the time at the start of the $k$th episode, and we set $t_{m+1} = T+1$.
Analogously, for each $t\in[T]$, we write $k_t$ to denote the current episode at time $t$.  
$\widehat{Q}_k$ is the Q-estimate used during the $k$th episode, and $\widehat{V}_k = \clipH(MQ)$.
We will frequently use the fact that $\spannorm{\widehat{V}_k} \leq H$ for all $k\in[m]$.
$N_k(s,a,s')$ and $N_k(s,a)$ are the counts of $(s,a,s')$ and $(s,a)$, respectively, at the start of the $k$th episode, and we let $N_k$ denote the length of the $k$th episode.
We write $n_k(s,a)$ as shorthand for $\max \left\{ N_k(s,a), 1 \right\}$.
For any $(s,a)\in\St\times\A, V\in\R^\St$, and $k\in[m]$, we let $b_k(s,a,V)$ denote the bonus used for the $k$th episode, namely
\[
	b_k(s,a,V) \coloneq \max \left\{ 4\sqrt{\frac{\V\left( \widehat{P}^k_{s,a}, V \right)U}{ n_k(s,a) }}, 32 \frac{HU}{n_k(s,a)} \right\}.
\]
For any $(s,a)\in\St\times\A$ and $V\in\R^\St$ we define
\[
    \V_{s,a}(V) \coloneq \V(P_{s,a},V).
\]
Further recall that $\delta' = \frac{\delta}{9S^2AT}$ and $U = \log\left( \frac{1}{\delta'} \right)$.
It is easy to see by the doubling trick of Algorithm~\ref{alg:oovi} that the number of episodes is bounded as $m \leq SAU$. 
We also have $\varepsilon_k = \frac{1}{t_k(1-\gamma)}$.

\begin{proof}
Let $T\geq 1$ and $\gamma, \delta\in (0,1)$ be arbitrary.
For the optimal discounted value function we will drop the $\gamma$ and simply write $\vstar$.
By assumption $H \geq \spannorm{\vstar}$.

As is standard in the analysis of optimistic algorithms for online RL, a key step in our analysis is to establish an optimism property, which enables the rest of our regret decomposition.
\paragraph{Step 1: Optimism}
We state the fact that with high probability, our $Q$-estimate $\widehat{Q}_k$ and value estimate $\widehat{V}_k$ are indeed optimistic across all episodes. 
We defer the proof to Appendix~\ref{sec:optimism_proof}.
\begin{lemma}[Optimism] \label{lem:optimism}
With probability $1-SAT\delta'$, both of the following hold:
\begin{enumerate}
	\item $\widehat{Q}_k(s,a) \geq \qstar(s,a) - \varepsilon_k$ and $\widehat{V}_k(s) \geq \vstar(s) - \varepsilon_k$ for all $(s,a,k) \in\St\times\A\times[m]$.
	\item For any $k\in[m]$ and $t\in\left\{ t_k,\dots,t_{k+1}-1 \right\}$, $\widehat{V}_k(s_t) \leq r(s_t,a_t) + \gamma \widehat{P}^k_{s_t,a_t} \widehat{V}_k + b_k\left( s_t,a_t, \widehat{V}_k \right)$. 
\end{enumerate}

\end{lemma}
Equipped with Lemma~\ref{lem:optimism}, we turn to decomposing the regret.

\paragraph{Step 2: Regret Decomposition}
Under the successful events of Lemma~\ref{lem:optimism}, observe that we have the following bound for any $t\in\left\{ t_k,\dots,t_{k+1}-1 \right\}$:
\begin{equation} \label{eq:single_reg}
(1-\gamma)\vstar(s_t) - r(s_t,a_t)
\leq
\gamma\left( \widehat{P}^k_{s_t,a_t} \widehat{V}_k - \widehat{V}_k(s_t) \right) + \gamma b_k\left( s_t, a_t, \widehat{V}_k \right) + (1-\gamma)\varepsilon_k.
\end{equation}
Now, it is not immediately obvious how we should bound the first term of the RHS, so we will relate it to something we do know how to bound.
Observe that this term vaguely looks like $P_{s_t,a_t} \widehat{V}_k - \widehat{V}_k(s_{t+1})$.
Defining $X_t$ to be this expression, one can verify that $\{X_t\}_{0 \leq t \leq T}$ is a martingale difference sequence with respect to the filtration $\mathcal{F}_t = \sigma(s_1,a_1,\dots,s_{t+1},a_{t+1})$.
Consequently, we can bound $\sum_{t=1}^{T} X_t$ with the following martingale concentration result.

\begin{lemma}[Martingale Concentration; Adapted from Lemma 13 in \citealt{zhang_is_2021}] \label{lem:mart_conc}
	Let $\{M_n\}_{n\geq 0}$ be a martingale with respect to some filtration $\{\mathcal{F}_n\}_{n\geq 0}$ such that $M_0 = 0$ and $|M_n - M_{n-1}| \leq c$ almost surely for some  $c \geq 0$ and all $n \geq 1$.
	Let  $\var_n = \sum_{k=1}^{n} \mathbb{E}\left[ (M_k - M_{k-1})^2 \mid \mathcal{F}_{k-1} \right]$.
	Then for any positive integer $n$ and any $\delta \in (0,1)$, we have with probability at least $1 - 3n \delta$ that
	 \[
		 |M_n| < 2 \sqrt{2} \sqrt{\var_n\log\left( \frac{1}{\delta} \right)} + 4c\log\left( \frac{1}{\delta} \right) 
	.\] 
\end{lemma}

We now have motivation for transforming $\widehat{P}^k_{s_t,a_t} \widehat{V}_k - \widehat{V}_k(s_t)$ into $X_t$.
We accomplish this feat by introducing an additional term:
\[
	\widehat{P}^k_{s_t,a_t} \widehat{V}_k - \widehat{V}_k(s_t)
	=
	\left(\widehat{P}^k_{s_t,a_t} \widehat{V}_k - \widehat{V}_k(s_{t+1})\right) + \left( \widehat{V}_k(s_{t+1}) - \widehat{V}_k(s_t) \right).
\]
Substituting back into \eqref{eq:single_reg} and summing over $t$ gives us the following decomposition:
\begin{align*}
	&\sum_{t=1}^{T} \left( (1-\gamma) \vstar(s_t) - r(s_t,a_t) \right) \\
	&= \sum_{k=1}^{m} \sum_{t=t_k}^{t_{k+1}-1} \left( (1-\gamma) \vstar(s_t) - r(s_t,a_t) \right) \\
	&\leq \sum_{k=1}^{m} \sum_{t=t_k}^{t_{k+1}-1} \left( \widehat{P}^k_{s_t,a_t} \widehat{V}_k - \widehat{V}_k(s_t) + b_k\left( s_t,a_t,\widehat{V}_k \right) + (1-\gamma)\varepsilon_k \right) \\
	&= \underbrace{\sum_{k=1}^{m} \sum_{t=t_k}^{t_{k+1}-1} \left( \left( \widehat{P}^k_{s_t,a_t} - P_{s_t,a_t} \right) \widehat{V}_k + b_k\left( s_t,a_t,\widehat{V}_k \right) + (1-\gamma) \varepsilon_k \right)}_{\eqcolon \tmodel} \\
	&\qquad\qquad + \underbrace{\sum_{k=1}^{m} \sum_{t=t_k}^{t_{k+1}-1} \left( P_{s_t,a_t}\widehat{V}_k - \widehat{V}_k(s_{t+1}) \right)}_{\eqcolon \tmart} \\
	&\qquad\qquad + \underbrace{\sum_{k=1}^{m} \sum_{t=t_k}^{t_{k+1}-1} \left( \widehat{V}_k(s_{t+1}) - \widehat{V}_k(s_t) \right)}_{\eqcolon \tind}.
\end{align*}

We call the first term $\tmodel$ because it is a function of the error of our model estimate.
Note that we include the bonus term in $\tmodel$ mainly due to technical reasons, but intuitively our bound for the model error term will involve a Bernstein inequality and this will absorb the Bernstein-style bonus.
The second term is $\tmart$ because we will bound it using martingale concentration.
Lastly, we call the third term $\tind$ because it arises due to our need to shift indices.

\paragraph{Step 3: Bounding $\tmodel$ and $\tmart$ by cumulative variance terms}
Our next step is to bound each term in the decomposition.
Before doing so, we introduce some cumulative variance terms that arise in the analysis.
We define
\begin{equation*}
\varstar := \sum_{t=1}^{T} \varstat{\vstar},
\qquad \qquad
\vardiff := \sum_{t=1}^{T} \varstat{\widehat{V}_{k_t} - \vstar}.
\end{equation*}

We start with $\tmodel = \sum_{k=1}^{m} \sum_{t=t_k}^{t_{k+1}-1} \left(\left(\widehat{P}^k_{s_t,a_t} - P_{s_t,a_t}\right) \widehat{V}_k + b_k\left(s_t,a_t, \widehat{V}_k \right) + (1-\gamma)\varepsilon_k\right)$.
We would like to bound the term involving model error via a Bernstein-like concentration inequality.
We cannot immediately do so, however, because $\widehat{P}^k_{s_t,a_t}$ and $\widehat{V}_k$ are not statistically independent.
To overcome this hurdle, we perform the following decomposition:
\[
\left(\widehat{P}^k_{s_t,a_t} - P_{s_t,a_t}\right) \widehat{V}_k
=
\left(\widehat{P}^k_{s_t,a_t} - P_{s_t,a_t}\right)\vstar + \left(\widehat{P}^k_{s_t,a_t} - P_{s_t,a_t}\right)\left( \widehat{V}_k - \vstar \right)  
.\] 
Since $\vstar$ is fixed, we can apply Bernstein to the first term.
Bounding the second term is also doable because $\left\|\widehat{V}_k - \vstar\right\|_\infty$ decreases as $k$ increases.
The bonus term is designed so that it is subsumed by the model error term.
We state the bound for $\tmodel$ in the following lemma, which we formally prove in Appendix~\ref{sec:t_model_proof}.
\begin{lemma} \label{lem:t_model}
	With probability at least $1-2S^2AT\delta'$, we have
	\[
	\tmodel \leq \bigo\left( \sqrt{SAU^2 \varstar} + \sqrt{\Gamma SAU^2\vardiff} + \Gamma HSAU^2 \right) 
	.\] 
\end{lemma}

We now move on to $\tmart$.
As discussed above, we can apply Lemma \ref{lem:mart_conc}, which gives us that with probability $1 - 3T\delta'$,
\begin{equation} \label{eq:t_mart}
	\tmart \leq 2\sqrt{2} \sqrt{\sum_{t=1}^{T} \varstat{\widehat{V}_{k_t}} U} + 4HU 
	\leq 4 \sqrt{\varstar U} + 4\sqrt{\vardiff U}  + 4HU,
\end{equation}
where the second inequality holds because
\begin{align*}	
\sqrt{ \sum_{t=1}^{T} \varstat{ \widehat{V}_{k_t}}}
= \sqrt{\sum_{t=1}^{T} \V_{s_t,a_t}\left(\widehat{V}_{k_t} - \vstar + \vstar\right)}
&\leq \sqrt{\sum_{t=1}^{T} \left(2\V_{s_t,a_t}\left(\widehat{V}_{k_t} - \vstar \right) + 2\V_{s_t,a_t}(\vstar)\right)}\\
& = \sqrt{2\varstar + 2\vardiff}
\leq \sqrt{2 \varstar} + \sqrt{2 \vardiff}.  
\end{align*}

\paragraph{Step 4: Bounding cumulative variance terms}
We can bound $\vardiff$ with the following lemma, whose proof we defer to Appendix~\ref{sec:var_bounds_proof}.
\begin{lemma} \label{lem:var_bounds}
Conditioned on the successful events of Lemma \ref{lem:optimism}, we have with probability at least $1-3T\delta'$ that
\[
\vardiff \leq \bigo\left(\tmodel H + H^2SAU\right)
.\] 
\end{lemma}
Subsequently, under the successful events of Lemma \ref{lem:t_model}, \eqref{eq:t_mart}, and Lemma \ref{lem:var_bounds}, we have
\[
	\tmodel + \tmart
	\leq
	\bigo\left( \sqrt{\Gamma HSAU^2} \sqrt{\tmodel + \tmart} + \sqrt{\varstar SAU^2} + \Gamma HSAU^2 \right),
\]
which by Lemma \ref{lem:ab} implies that
\begin{equation} \label{eq:mart_model_bound}
	\tmodel + \tmart \leq \mathcal{O}\left(\sqrt{\varstar SAU^2} + \Gamma HSAU^2 \right).
\end{equation}

\paragraph{Step 5: Bounding $\tind$}
It remains to bound $\tind$, which turns out to be the most straightfoward argument.
We compute
\[
\tind
= \sum_{k=1}^{m} \sum_{t=t_k}^{t_{k+1}-1} \left( \widehat{V}_k(s_{t+1}) - \widehat{V}_k(s_t) \right) 
= \sum_{k=1}^{m} \left( \widehat{V}_k(s_{t_{k+1}}) - \widehat{V}_k(s_{t_k}) \right)
\leq mH
\leq HSAU.
\]

\paragraph{Step 6: Recombining}
Finally, we recombine terms.
Conditioned under the successful events of Lemma \ref{lem:optimism} and \eqref{eq:mart_model_bound}, we have
\begin{equation*}
\reg_\gamma(T)
\leq \gamma(\tmodel + \tmart + \tind)
\leq \mathcal{O}\left( \sqrt{\varstar SAU^2} + \Gamma HSAU^2 \right).
\end{equation*}
Via a union bound (over the high probability events of Lemma~\ref{lem:optimism}, Lemma~\ref{lem:t_model}, \eqref{eq:t_mart}, and Lemma~\ref{lem:var_bounds}), this occurs with probability at least $1 - 9S^2AT\delta' = 1-\delta$.
\end{proof}

\section{Proof of Lemma~\ref{lem:var_star_bound}} \label{sec:var_star_bound_proof}

\begin{proof}
Let $\delta \in (0,1)$ be arbitrary, and set $\delta'=\frac{\delta}{6T}$.
For the optimal value function we will drop $\gamma$ and simply write $\vstar$.
Our goal is to bound $\varstar = \sum_{t=1}^{T} \V(P_{s_t,a_t},\vstar)$.

Now, observe that if we set $\vtilstar \coloneq \vstar - \min_s \vstar(s)$, then we have $\vtilstar \in \left[ 0,\spannorm{\vstar} \right]^S$ and $\sum_{t=1}^{T} \varstat{\vstar} = \sum_{t=1}^{T} \varstat{\vtilstar}$.
Subsequently, we perform the following decomposition.
\begin{align*}
\sum_{t=1}^{T}  \varstat{\vtilstar}
&= \sum_{t=1}^{T} \left( P_{s_t,a_t}\left( \vtilstar \right)^2 - \left( P_{s_t,a_t} \vtilstar \right)^2 \right) \\
&=
\underbrace{\sum_{t=1}^{T} \left( P_{s_t,a_t}\left(\vtilstar\right)^2 - \left(\vtilstar(s_{t+1})\right)^2 \right)}_{\eqcolon \mathscr{T}_1^\star}
+ 
\underbrace{\sum_{t=1}^{T} \left( \left(\vtilstar(s_{t})\right)^2 - \left(P_{s_t,a_t}\vtilstar\right)^2  \right)}_{\eqcolon \mathscr{T}_2^\star} \\
&\qquad\qquad\qquad\qquad\qquad\qquad\qquad\qquad\qquad\qquad + 
\underbrace{\sum_{t=1}^{T} \left( \left(\vtilstar(s_{t+1})\right)^2 - \left(\vtilstar(s_t)\right)^2 \right)}_{\eqcolon \mathscr{T}_3^\star}.
\end{align*}
With probability at least $1 - 3T\delta'$, we have that
\begin{align*}
\mathscr{T}_1^\star 
&\leq 2\sqrt{2}\sqrt{ \sum_{t=1}^{T} \varstat{ \left(\vtilstar\right)^2 } \log\left( \frac{1}{\delta'}\right)} + 4\spannorm{\vstar}^2\log\left(\frac{1}{\delta'}\right) \\
&\leq 4\sqrt{2}\spannorm{\vstar}\sqrt{\varstar \log\left(\frac{1}{\delta'}\right)} + 4\spannorm{\vstar}^2\log\left(\frac{1}{\delta'}\right),
\end{align*}
where the first inequality holds by Lemma \ref{lem:mart_conc} and the second inequality holds by Lemma \ref{lem:varsq}.

Next, we compute that with probability at least $1 - 3T\delta'$, we have
\begin{align*}
\mathscr{T}_2^\star
&= \sum_{t=1}^{T} \left( \left(\vtilstar(s_{t})\right)^2 - \left( P_{s_t,a_t}\vtilstar \right)^2 \right) \\
&\stackrel{(i)}{\leq} 2\spannorm{\vstar} \sum_{t=1}^{T}  \max \left\{ \vtilstar(s_t) - P_{s_t,a_t}\vtilstar, 0 \right\} \\
&\leq 2\spannorm{\vstar} \sum_{t=1}^{T} \max \left\{ \vtilstar(s_t) - P_{s_t,a_t}\vtilstar + 1, 0 \right\}   \\
&\stackrel{(ii)}{=} 2\spannorm{\vstar} \sum_{t=1}^{T} \left(\vtilstar(s_{t}) - P_{s_{t},a_{t}}\vtilstar + 1 \right) \\
&= 2\spannorm{\vstar} \left(T + \sum_{t=1}^{T} \left( \vtilstar(s_t) - P_{s_t,a_t}\vtilstar \right)\right)  \\
&\leq 2\spannorm{\vstar} \left(T + \vtilstar(s_1) + \sum_{t=1}^{T} \left( \vtilstar(s_{t+1}) - P_{s_t,a_t}\vtilstar \right)\right)  \\
&\stackrel{(iii)}{\leq} 2\spannorm{\vstar} \left( T + \vtilstar(s_1) + 2\sqrt{2}\sqrt{\varstar \log\left(\frac{1}{\delta'}\right)} + 4\spannorm{\vstar}\log\left(\frac{1}{\delta'}\right)  \right) \\
&\leq 4\sqrt{2}\spannorm{\vstar}\sqrt{\varstar \log\left(\frac{1}{\delta'}\right)} +  2T\spannorm{\vstar} + 10\spannorm{\vstar}^2\log\left(\frac{1}{\delta'}\right).
\end{align*}
Inequality $(i)$ holds because  $a^2-b^2=(a+b)(a-b) \leq (a+b)\max \{a-b,0\} \leq 2\spannorm{\vstar} \max \{a-b, 0\}$ for $a,b \in \left[ 0,\spannorm{\vstar} \right]$.
Equality $(ii)$ holds because $\vtilstar(s_t) - P_{s_t,a_t}\vtilstar + 1 \geq 0$.
Indeed, we have
\[
	\vtilstar(s_t) - P_{s_t,a_t}\vtilstar
	= \vstar(s_t) - P_{s_t,a_t}\vstar 
	\geq \qstar(s_t,a_t) - P_{s_t,a_t}\vstar 
	= r(s_t,a_t) - (1-\gamma) P_{s_t,a_t}\vstar 
	\geq -1
.\]
Finally, inequality $(iii)$ holds with probability at least $1-3T\delta'$  by Lemma \ref{lem:mart_conc}.

Observing that $\mathscr{T}_{3}^\star$ is a telescoping sum, we have
\begin{align*}
\mathscr{T}_3^\star
&= \sum_{t=1}^{T} \left( \left(\vtilstar(s_{t+1})\right)^2 - \left(\vtilstar(s_t)\right)^2 \right) \\
&\leq \left( \vtilstar(s_{T+1}) \right)^2 \\
&\leq \spannorm{\vstar}^2.
\end{align*}

Recombining terms, we have with probability $1 - 6T\delta'$ that
\[
	\varstar \leq \mathscr{T}_1^\star + \mathscr{T}_2^\star + \mathscr{T}_3^\star \leq \bigo\left(\spannorm{\vstar}\sqrt{\varstar \log\left(\frac{1}{\delta'}\right)} + T\spannorm{\vstar} + \spannorm{\vstar}^2 \log\left(\frac{1}{\delta'}\right)\right)
,\]
which by Lemma \ref{lem:ab} implies
\[
	\varstar \leq \bigo\left(\spannorm{\vstar}T + \spannorm{\vstar}^2\log\left(\frac{1}{\delta'}\right)\right)
.\] 
Substituting back $\delta' = \frac{\delta}{6T}$ gives us that
\[
	\varstar \leq \bigo\left(\spannorm{\vstar}T + \spannorm{\vstar}^2\log(T/\delta)\right)
\] 
with probability at least $1-\delta$.
\end{proof}

\section{Proof of Lemma~\ref{lem:avg_to_disc_reduction}} \label{sec:avg_to_disc_reduction_proof}

The proof relies on the following lemma.

\begin{lemma}[Lemma 6 in \citealt{zurek_span-agnostic_2025}] \label{lem:avg_to_disc}
	Suppose the underlying MDP is weakly communicating.
	Let $\gamma\in(0,1)$.
	The optimal value function $\vstar_\gamma$ satisfies
	\[
		\left\| \rhostar - (1-\gamma) \vstar_\gamma \right\|_\infty \leq (1-\gamma)\spannorm{\vstar_\gamma}.
	\]
\end{lemma}

\begin{proof}
For any $T\geq 1$ and $\gamma\in(0,1)$, we have
\begin{align*}
\reg(T)
&=\sum_{t=1}^{T} \left( \rhostar - r(s_t,a_t) \right) \\
&= \sum_{t=1}^{T}\left( \rhostar - (1-\gamma)\vstar_\gamma(s_t)\right)  + \sum_{t=1}^{T} \left( (1-\gamma)\vstar_\gamma(s_t) - r(s_t,a_t)\right) \\
&\leq \sum_{t=1}^{T}\left( 1-\gamma \right)\spannorm{\vstar_\gamma}  + \sum_{t=1}^{T} \left( (1-\gamma)\vstar_\gamma(s_t) - r(s_t,a_t)\right) \\
&= (1-\gamma)\spannorm{\vstar_\gamma}T + \reg_\gamma(T),
\end{align*}
where the inequality is due to Lemma~\ref{lem:avg_to_disc}.
\end{proof}

\section{Missing Proofs from Appendix~\ref{sec:var_dependent_gamreg_proof}}

\subsection{Proof of Lemma~\ref{lem:optimism}} \label{sec:optimism_proof}

\begin{proof}
We denote by $\widehat{\mathcal{T}}_k$ the empirical Bellman operator used in Algorithm~\ref{alg:oovi} during the $k$th episode.
In particular, we have
\[
	\left(\widehat{\mathcal{T}}_kQ\right)(s,a) = r(s,a) + \gamma \widehat{P}_{s,a}^k \clipH (M Q) + \gamma b_k\left( s, a, \clipH (MQ) \right).
\]
We also write $\text{iters}_k$ to be the number of value iterations used Algorithm~\ref{alg:oovi} during the $k$th episode, so that
\[
    \text{iters}_k = \left\lceil \frac{1}{1-\gamma} \log \frac{1+ 32HU}{\varepsilon_k(1-\gamma)} \right\rceil.
\]
Observe that 
\[
\gamma^{\itersk}
\leq
\exp(-(1-\gamma)\itersk)
\leq
\frac{\varepsilon_k(1-\gamma)}{1+32HU},
\]
where the first inequality is due to $x\leq e^{-(1-x)}$.
The following Lemma, which states some crucial properties of $\widehat{\mathcal{T}}_k$, is proved in Appendix~\ref{sec:tkhat_properties_proof}.
\begin{lemma} \label{lem:tkhat_properties}
	With probability $1-SAT\delta'$, the following hold for all $k\in[m]$.
	\begin{enumerate}
		\item $\widehat{\mathcal{T}}_k$ is monotonic: $\widehat{\mathcal{T}}_k Q \geq \widehat{\mathcal{T}}_k Q'$ for any $Q,Q'\in\R^{\St\times\A}$ such that $Q\geq Q'$.   
		\item $\widehat{\mathcal{T}}_k$ is a $\gamma$-contraction: $\left\| \widehat{\mathcal{T}}_kQ - \widehat{\mathcal{T}}_kQ' \right\|_\infty \leq \gamma \left\| Q - Q' \right\|_\infty$ for any $Q,Q'\in\R^{\St\times\A}$.  
		\item $\qstar \leq \widehat{\mathcal{T}}_k \qstar$.
	\end{enumerate}	
\end{lemma}

Now, assume that the events of Lemma~\ref{lem:tkhat_properties} hold, and fix an arbitrary $k$.
We first prove Part 1 of Lemma~\ref{lem:optimism}.
By the Banach fixed point theorem \citep{pugh_real_2015}, the fact that $\widehat{\mathcal{T}}_k$ is a $\gamma$-contraction implies that $\widehat{\mathcal{T}}_k$ has a unique fixed point, which we will denote $\qhatstar_k$. 
By monotonicity of $\widehat{\mathcal{T}}_k$, the sequence $\qstar, \widehat{\mathcal{T}}_k \qstar, \widehat{\mathcal{T}}_k \widehat{\mathcal{T}}_k \qstar, \dots$ is nondecreasing and converges to $\widehat{Q}^\star_k$, which implies
\begin{equation} \label{eq:fixed_point_optimistic}
    \qstar \leq \qhatstar_k.
\end{equation}
Furthermore, we have
\begin{equation} \label{eq:fixed_point_epsilon}
	\left\| \widehat{Q}_k - \qhatstar_k \right\|_\infty
	=
	\left\| \left( \widehat{\mathcal{T}}_k \right)^{\left( \text{iters}_k \right)} \mathbf{0} - \left( \widehat{\mathcal{T}}_k \right)^{\left( \text{iters}_k \right)}\qhatstar_k \right\|_\infty
	\leq
	\gamma^{\text{iters}_k}\left\| \qhatstar_k \right\|_\infty
	\leq
	\varepsilon_k,
\end{equation}
where the final inequality holds by the above bound on $\gamma^\itersk$ and
\begin{align*}
	&\left( \widehat{\mathcal{T}}_k Q \right)(s,a) \leq 1 + \gamma \|Q\|_\infty + \gamma 32HU \\
	&\quad\implies
	\left\| \widehat{\mathcal{T}}_k \mathbf{0} \right\|_\infty \leq 1 + \gamma 32 HU \\
	&\quad\implies
	\left\| \widehat{\mathcal{T}}_k \widehat{\mathcal{T}}_k \mathbf{0} \right\|_\infty \leq 1 + \gamma 32HU + \gamma(1 + \gamma 32 HU) \\
	&\qquad\vdots \\
	&\quad\implies
    \left\| \qhatstar_k \right\|_\infty \leq \frac{1 + \gamma 32HU}{1 - \gamma}.
\end{align*}

Combining \eqref{eq:fixed_point_optimistic} and \eqref{eq:fixed_point_epsilon} gives us that (elementwise)
\[
	\widehat{Q}_k \geq \qstar - \mathbf{1}\varepsilon_k.
\]
Subsequently, for any $s\in\St$, 
\begin{align*}
	\widehat{V}_k(s) 
	&= \min \left\{ \left(M\widehat{Q}_k\right)(s), \min_{s'} \left( M \widehat{Q}_k \right)(s') + H \right\} \\
	&\geq \min \left\{ \left(M\left( \qstar - \mathbf{1}\varepsilon_k \right)\right)(s), \min_{s'} \left( M \left( \qstar-\mathbf{1}\varepsilon_k \right) \right)(s') + H \right\} \\ 
	&= \min \left\{ \left(M\qstar \right)(s), \min_{s'} \left( M \qstar \right)(s') + H \right\} - \varepsilon_k \\
	&= \min \left\{ \vstar(s), \min_{s'} \vstar(s') + H \right\} - \varepsilon_k \\
	&= \vstar(s) - \varepsilon_k,
\end{align*}
so the desired result holds.

It remains to prove Part 2 of Lemma~\ref{lem:optimism}.
We remark that $\widehat{Q}_k \leq \widehat{\mathcal{T}}_k \widehat{Q}_k$ because monotonicity implies
\[
	\mathbf{0} 
	\leq \widehat{\mathcal{T}}_k \mathbf{0}
	\leq \widehat{\mathcal{T}}_k \widehat{\mathcal{T}}_k \mathbf{0}
	\leq \dots
	\leq \underbrace{\left( \widehat{\mathcal{T}}_k \right)^{\left( \text{iters}_k \right)} \mathbf{0}}_{= \widehat{Q}_k}	
	\leq \underbrace{\left( \widehat{\mathcal{T}}_k \right)^{\left( \text{iters}_k + 1 \right)} \mathbf{0}}_{= \widehat{\mathcal{T}}_k \widehat{Q}_k}.
\]
It follows that for any $t\in\left\{ t_k,\dots,t_{k+1}-1 \right\}$, 
\begin{align*}
	\widehat{V}_k(s_t)
	&\leq \left( M \widehat{Q}_k \right)(s_t) \\
	&= \widehat{Q}_k(s_t,a_t) \\
	&\leq \widehat{\mathcal{T}}_k \widehat{Q}_k(s_t,a_t) \\
	&= r(s_t,a_t) + \gamma \widehat{P}^k_{s_t,a_t}\widehat{V}_k + \gamma b_k\left( s_t,a_t,\widehat{V}_k \right),
\end{align*}
so the desired result holds.
\end{proof}

\subsection{Proof of Lemma \ref{lem:t_model}} \label{sec:t_model_proof}

\begin{proof}
As explained in the main section of the proof, we decompose
\begin{align*}
    \tmodel
	&= \sum_{k=1}^{m} \sum_{t=t_k}^{t_{k+1}-1}\left(\left( \widehat{P}_{s_t,a_t}^k - P_{s_t,a_t} \right)\widehat{V}_k + b_k\left(s_t,a_t, \widehat{V}_k\right) + (1-\gamma)\varepsilon_k\right) \\
	&= \sum_{k=1}^{m}\sum_{t=t_k}^{t_{k+1}-1}\bigg(\left( \widehat{P}_{s_t,a_t}^k - P_{s_t,a_t} \right)\vstar + \left( P_{s_t,a_t}^k - P_{s_t,a_t} \right)\left( \widehat{V}_k - \vstar \right) \\
    & \qquad + b_k\left(s_t,a_t, \widehat{V}_k\right) + (1-\gamma)\varepsilon_k \bigg).
\end{align*}
We will bound each of the first three terms inside the sum separately under high probability events.
Recall that for ease of notation we denote $n_k(s,a) \coloneq \max\left\{ N_k(s,a),1 \right\}.$ 
Let $\mathcal{E}_1$ be the event that
\[
	\left| \widehat{P}^k_{s,a,s'} - P_{s,a,s'} \right| \leq \sqrt{\frac{2P_{s,a,s'}U}{n_k(s,a)}} + \frac{\ind\left(P_{s,a,s'} > 0\right)U}{3 n_k(s,a)} \qquad \forall (s,a,s',k)\in\St\times\A\times\St\times[m],
\]
and let $\mathcal{E}_2$ be the event that
\begin{equation} \label{eq:t_model_term1_bound}
	\left|\left(\widehat{P}_{s,a}^k - P_{s,a}\right) \vstar\right| \leq \sqrt{\frac{2\varsa{\vstar}U}{n_k(s,a)}} + \frac{HU}{3 n_k(s,a)} \qquad \forall (s,a,k)\in\St\times\A\times[m].
\end{equation}
We will later confirm that $\mathcal{E}_1$ and $\mathcal{E}_2$ are indeed high probability events.
Under event $\mathcal{E}_1$, for any $(s,a,k)\in \St\times\A\times[m]$ we have
\begin{equation} \label{eq:t_model_term2_bound}
\begin{aligned}[b]
&\left(\widehat{P}_{s,a}^k - P_{s,a}\right)\left(\widehat{V}_k - \vstar\right) \\
&= \sum_{s' \in \mathcal{S}} \left( \widehat{P}_{s,a,s'}^k - P_{s,a,s'} \right) \left( \widehat{V}_k(s') - \vstar(s') \right) \\
&\stackrel{(i)}{=} \sum_{s' \in \mathcal{S}} \left( \widehat{P}_{s,a,s'}^k - P_{s,a,s'} \right) \left( \widehat{V}_k(s') - \vstar(s') - P_{s,a}(\widehat{V}_k - \vstar) \right) \\
&\stackrel{(ii)}{\leq} \sum_{s'\in \mathcal{S}} \left| \sqrt{\frac{2 P_{s,a,s'} U}{n_k(s,a)}} + \frac{\ind\left(P_{s,a,s'}>0\right)U}{3n_k(s,a)} \right| \cdot \left| \widehat{V}_k(s') - \vstar(s') - P_{s,a}\left(\widehat{V}_k - \vstar\right) \right| \\ 
&\stackrel{(iii)}{\leq} \sum_{s'\in \mathcal{S}} \sqrt{\frac{2 P_{s,a,s'} U}{n_k(s,a)}} \cdot \left| \widehat{V}_k(s') - \vstar(s') - P_{s,a}\left(\widehat{V}_k - \vstar\right) \right| + \frac{2\Gamma HU}{3n_k(s,a)} \\
&\stackrel{(iv)}{\leq} \sqrt{\frac{2\Gamma U \sum_{s'\in\St} P_{s,a,s'} \left( \widehat{V}_k(s') - \vstar(s') - P_{s,a}\left( \widehat{V}_k - \vstar \right) \right)^2}{n_k(s,a)}} + \frac{2\Gamma HU}{3n_k(s,a)} \\
&= \sqrt{\frac{2\Gamma U \varsa{\widehat{V}_k - \vstar}}{n_k(s,a)}} + \frac{2\Gamma HU}{3n_k(s,a)}.
\end{aligned}
\end{equation}
Equality $(i)$ holds because $\sum_{s'} \left( \widehat{P}^k_{s,a,s'} - P_{s,a,s'} \right) c = 0$ for any $c \in \R$.
Inequality $(ii)$ holds under event $\mathcal{E}_1$.
We obtain inequality $(iii)$ by bounding $\left| \widehat{V}_k(s') - \vstar(s') - P_{s,a}\left( \widehat{V}_k - \vstar \right)\right| \leq 2H$ and summing over $s'$.
Inequality $(iv)$ follows from Cauchy-Schwarz. 
Moreover, under $\mathcal{E}_1$, for any $(s,a,k)\in\St\times\A\times[m]$ we have
\begin{equation} \label{eq:t_model_bonus_bound}
\begin{aligned}[b]
    b_k\left(s,a, \widehat{V}_k\right)
	&= \max\left\{4\sqrt{\frac{\V\left( \widehat{P}^k_{s,a},\widehat{V}_k \right)U}{n_k(s,a)}}, 32 \frac{HU}{n_k(s,a)}\right\} \\
	&\leq 4\sqrt{\frac{\V\left( \widehat{P}^k_{s,a},\widehat{V}_k \right)U}{n_k(s,a)}} + 32 \frac{HU}{n_k(s,a)} \\
	&\leq 4\sqrt{3/2} \sqrt{\frac{\V_{s,a}\left( \widehat{V}_k \right)U}{n_k(s,a)}} + 4\sqrt{4/3}\frac{\sqrt{\Gamma H^2U}}{n_k(s,a)} + 32 \frac{HU}{n_k(s,a)} \\
	&\leq 5 \sqrt{\frac{\V_{s,a}\left( \widehat{V}_k \right)U}{n_k(s,a)}} + 37 \frac{\Gamma HU}{n_k(s,a)},
\end{aligned}
\end{equation}
with the second inequality holding due to
\begin{align*}
\V\left( \widehat{P}^k_{s,a}, \widehat{V}_k \right) 
&= \sum_{s'\in \mathcal{S}}\widehat{P}^k_{s,a,s'} \left( \widehat{V}_k(s') - \widehat{P}^k_{s,a}\widehat{V}_k \right)^2 \\
&\stackrel{(i)}{\leq} \sum_{s'\in \mathcal{S}} \widehat{P}^k_{s,a,s'}\left( \widehat{V}_k(s') - P_{s,a}\widehat{V}_k \right)^2 \\
&\stackrel{(ii)}{\leq} \sum_{s' \in \mathcal{S}} \left( \frac{3}{2} P_{s,a,s'} + \frac{4\ind\left(P_{s,a,s'}>0\right)U}{3n_k(s,a)} \right) \left( \widehat{V}_k(s') - P_{s,a}\widehat{V}_k \right)^2 \\
&\leq \frac{3}{2} \V_{s,a}\left( \widehat{V}_k \right) + \frac{4\Gamma H^2 U}{3n_k(s,a)}. 
\end{align*}
Here, inequality $(i)$ is because $\mathbb{E}[X]$ minimizes $f(\lambda) = \mathbb{E}[(X-\lambda)^2]$,
and inequality $(ii)$ holds under $\mathcal{E}_1$.
Indeed, under $\mathcal{E}_1$, we have
\begin{align*}
& \left| \widehat{P}^k_{s,a,s'} - P_{s,a,s'} \right| \leq \sqrt{\frac{2P_{s,a,s'}\ind\left(P_{s,a,s'}>0\right)U}{n_k(s,a)}} + \frac{\ind\left(P_{s,a,s'}>0\right) U}{3n_k(s,a)} \\
&\implies \widehat{P}^k_{s,a,s'} - P_{s,a,s'} \leq \frac{1}{2}P_{s,a,s'} + \frac{\ind\left(P_{s,a,s'}>0\right)U}{n_k(s,a)} + \frac{\ind\left(P_{s,a,s'}>0\right)U}{3n_k(s,a)} \\
&\implies \widehat{P}^k_{s,a,s'} \leq \frac{3}{2}P_{s,a,s'} + \frac{4\ind\left(P_{s,a,s'}>0\right)U}{3n_k(s,a)},
\end{align*}
where the first implication holds because $\sqrt{ab} \leq \frac{a}{2} + \frac{b}{2}$ for $a,b\geq 0$.

Now, combining \eqref{eq:t_model_term1_bound}, \eqref{eq:t_model_term2_bound}, and \eqref{eq:t_model_bonus_bound} gives us that under $\mathcal{E}_1 \cap \mathcal{E}_2$, 
\[
	\tmodel \leq \bigo\left( \sum_{k=1}^{m} \sum_{t=t_k}^{t_{k+1}-1} \left(\sqrt{\frac{U\varstat{\vstar}}{n_k(s_t,a_t)}} + \sqrt{\frac{\Gamma U\V_{s_t,a_t}\left(\widehat{V}_k-\vstar\right)}{n_k(s_t,a_t)}} + \frac{\Gamma HU}{n_k(s_t,a_t)} + (1-\gamma)\varepsilon_k\right)\right).
\]
The following lemma, which we prove in Appendix~\ref{sec:wseq_proof}, allows us to easily bound the above.
\begin{lemma} \label{lem:wseq}
We have the following bounds.
\begin{enumerate}
	\item $\sum_{t=1}^{T} \frac{1}{n_{k_t}(s_t,a_t)} \leq SAU$.
	\item For nonnegative numbers $w_1,\dots,w_T$, we have $\sum_{t=1}^{T} \sqrt{\frac{w_t}{n_{k_t}(s_t,a_t)}} \leq \sqrt{SAU \sum_{t=1}^{T} w_t}$.
	\item $\sum_{t=1}^{T} (1-\gamma) \varepsilon_{k_t} \leq SAU$. 
\end{enumerate}
\end{lemma}
Setting $w_t = \V_{s_t,a_t}(\vstar)$, Part 2 of Lemma~\ref{lem:wseq} gives us that
\[
	\sum_{t=1}^{T} \sqrt{\frac{U\V_{s_t,a_t}(\vstar)}{n_{k_t}(s_t,a_t)}}
	\leq
	\sqrt{SAU^2 \varstar}.
\]
Next, setting $w_t = \V_{s_t,a_t}\left( \widehat{V}_k - \vstar \right)$, Part 2 of Lemma~\ref{lem:wseq} gives us that
\[
	\sum_{t=1}^{T} \sqrt{\frac{\Gamma U\V_{s_t,a_t}\left( \widehat{V}_k - \vstar \right)}{n_{k_t}(s_t,a_t)}}
	\leq
	\sqrt{\Gamma SAU^2 \vardiff}.
\]
Lastly, Parts 1 and 3 of Lemma~\ref{lem:wseq} give us that
\[
	\sum_{t=1}^{T} \left( \frac{\Gamma HU}{n_{k_t}(s_t,a_t)} + (1-\gamma)\varepsilon_{k_t} \right)
	\leq 2\Gamma HSAU^2.
\]
Combining these three bounds, we have that under $\mathcal{E}_1 \cap \mathcal{E}_2$,
\[
    \tmodel
	\leq
	\bigo\left( \sqrt{SAU^2 \varstar} + \sqrt{\Gamma SAU^2\vardiff} + \Gamma HSAU^2 \right).
\]

It remains to show that $\mathcal{E}_1$ and $\mathcal{E}_2$ hold with the claimed high probability.  
Starting with $\mathcal{E}_1$, fix $s,s'\in\St$ and $a\in\A$, and suppose $n_k(s,a) = n$ for some $n \geq 1$.   
Observe that $P^k_{s,a,s'} = \frac{1}{n}\sum_{i=1}^{n} Z_i$,
where $Z_1,\dots,Z_{n}$ are i.i.d. Bernoulli random variables with mean $P_{s,a,s'}$.
Denoting $Z$ to also be an independent Bernoulli random variable with mean  $P_{s,a,s'}$, we have with probability at least $1-\delta'$,
\[
	\left| \widehat{P}^k_{s,a,s'} - P_{s,a,s'} \right| \leq \sqrt{\frac{2\text{Var}(Z)\log(2 /\delta')}{n}} + \frac{\ind\left(P_{s,a,s'}>0\right)\log (2 / \delta')}{3n} \leq \sqrt{\frac{2P_{s,a,s'}U}{n}} + \frac{\ind\left(P_{s,a,s'}>0\right)U}{3n}  
,\] 
where the first inequality is due to Lemma~\ref{lem:bernstein} and the observation that $P_{s,a,s'} = 0 \implies \widehat{P}_{s,a,s'}^k = 0$, and the second inequality holds because $\text{Var}(Z) \leq \mathbb{E}[Z^2] = \mathbb{E}[Z] = P_{s,a,s'}$ and $\log(2 / \delta') \leq U$.
It follows from a union bound over possible $s,s',a,n$ gives us that $\mathcal{E}_1$ holds with probability at least $1 - S^2AT\delta'$.

Next, for $\mathcal{E}_2$, fix $s\in\St$ and $a\in\A$, and suppose $n_k(s,a) = n$ for some $n\geq 1$.
Observe that $\widehat{P}^k_{s,a}\vstar = \frac{1}{n} \sum_{i=1}^{n} Z_i$, where $Z_1,\dots,Z_n$ are i.i.d. multinoulli random variables that take the value $\vstar(s')$ with probability  $P_{s,a,s'}$ for each $s'$.
Letting $Z$ be i.i.d. from the same distribution as the  $Z_i$'s, we have that  $\mathbb{E}[Z] = P_{s,a}\vstar$ and 
\[
	\text{Var}(Z) = \mathbb{E}[Z^2] - (\mathbb{E}[Z])^2 = P_{s,a}(\vstar)^2 - (P_{s,a}\vstar)^2 = \V_{s,a}(\vstar)
.\]
Subsequently, Lemma \ref{lem:bernstein} gives us that with probability at least $1-\delta'$,
 \[
	 \left| P^t_{s,a}\vstar - P_{s,a}\vstar \right| \leq \sqrt{\frac{2\V_{s,a}(\vstar)U}{n}} + \frac{HU}{3n} 
,\]
so it follows from a union bound over possible $s,a,n$ that $\mathcal{E}_2$ holds with probability at least $1 - SAT\delta'$.
We conclude that $\mathcal{E}_1\cap \mathcal{E}_2$ occurs with probability at least $1-2S^2AT \delta'$, as desired. 
\end{proof}

\subsection{Proof of Lemma \ref{lem:var_bounds}} \label{sec:var_bounds_proof}

\begin{proof}
Our goal is to bound $\vardiff = \sum_{t=1}^{T} \V_{s_t,a_t}\left(\widehat{V}_{k_t} - \vstar\right)$.
We will proceed in a manner similar to the proof of Lemma~\ref{lem:var_star_bound}, where we bounded $\varstar$.
For ease of notation, write $D_k \coloneq \widehat{V}_k - \vstar$.
Then, set $\wt{D}_k \coloneq D_k - \min_s D_k(s)$ so that $\wt{D}_k \in \left[ 0,H \right]^{S}$ and $\vardiff = \sum_{t=1}^{T}\V_{s_t,a_t}\left( \wt{D}_{k_t} \right)$.  
We subsequently have
\begin{align*}
\sum_{t=1}^{T}  \V_{s_t,a_t}\left( \wt{D}_{k_t} \right) 
&= \sum_{k=1}^{m}\sum_{t=t_k}^{t_{k+1}-1} \left( P_{s_t,a_t}\left( \wt{D}_k \right)^2 - \left( P_{s_t,a_t} \wt{D}_k \right)^2 \right) \\
&=
\underbrace{\sum_{t=1}^{T} \left( P_{s_t,a_t}\left(\wt{D}_{k_t}\right)^2 - \left(\wt{D}_{k_t}(s_{t+1})\right)^2 \right)}_{\eqcolon \mathscr{T}_1^\text{diff}} \\
& \qquad\qquad\qquad +\underbrace{\sum_{k=1}^{m}\sum_{t=t_k}^{t_{k+1}-1} \left( \left(\wt{D}_{k}(s_{t})\right)^2 - \left(P_{s_t,a_t}\wt{D}_{k}\right)^2  \right)}_{\eqcolon \mathscr{T}_2^\text{diff}} \\
&\qquad\qquad\qquad + 
\underbrace{\sum_{k=1}^{m}\sum_{t=t_k}^{t_{k+1}-1} \left( \left(\wt{D}_k(s_{t+1})\right)^2 - \left(\wt{D}_k(s_t)\right)^2 \right)}_{\eqcolon \mathscr{T}_3^\text{diff}}.
\end{align*}
With probability at least $1 - 3T\delta'$, we have
\[
\mathscr{T}_1^\text{diff} 
\leq 2\sqrt{2}\sqrt{ \sum_{t=1}^{T} \V_{s_t,a_t}\left( \left(\wt{D}_{k_t}\right)^2 \right)U} + 4H^2U
\leq 4\sqrt{2}H\sqrt{\vardiff U} + 4H^2U
,\]
where the first inequality is by Lemma \ref{lem:mart_conc} and the second inequality is by Lemma \ref{lem:varsq}.

Next, we compute that
\begin{align*}
\mathscr{T}_2^\text{diff} 
&= \sum_{k=1}^{m}\sum_{t=t_k}^{t_{k+1}-1} \left( \left(\wt{D}_k(s_{t})\right)^2 - \left( P_{s_t,a_t}\wt{D}_k \right)^2 \right) \\
&\leq 2H \sum_{k=1}^{m}\sum_{t=t_k}^{t_{k+1}-1}  \max \left\{ \wt{D}_{k}(s_t) - P_{s_t,a_t}\wt{D}_k, 0 \right\} \\
&= 2H \sum_{k=1}^{m}\sum_{t=t_k}^{t_{k+1}-1} \max \left\{ D_k(s_t) - P_{s_t,a_t}D_k, 0 \right\} \\
&= 2H \sum_{k=1}^{m} \sum_{t=t_k}^{t_{k+1}-1} \max \left\{ \widehat{V}_k(s_t) - P_{s_t,a_t}\widehat{V}_k - \left( \vstar(s_t) - P_{s_t,a_t}\vstar \right), 0 \right\},
\end{align*}
with the inequality holding by $a^2 - b^2  \leq 2H\max \{a-b,0\}$ for $a, b \in \left[ 0,H \right]$.
Furthermore, we have
\begin{align}
	\widehat{V}_k(s_t) - P_{s_t,a_t}\widehat{V}_k
	&\leq r(s_t,a_t) + \gamma \widehat{P}^k_{s_t,a_t}\widehat{V}_k + \gamma b_k\left(s_t,a_t,\widehat{V}_k\right) - P_{s_t,a_t}\widehat{V}_k \nonumber \\
	&= r(s_t,a_t) - (1-\gamma) P_{s_t,a_t} \widehat{V}_k + \gamma\left( \widehat{P}^k_{s_t,a_t} - P_{s_t,a_t} \right) \widehat{V}_k + \gamma b_k\left(s_t,a_t,\widehat{V}_k\right) \label{eq:t2diff_emp_bellman}
\end{align}
as well as
\begin{equation} \label{eq:t2diff_bellman}
	\vstar(s_t) - P_{s_t,a_t}\vstar
	\geq
	r(s_t,a_t) - (1-\gamma) P_{s_t,a_t} \vstar.
\end{equation}
Continuing from above by plugging in \eqref{eq:t2diff_emp_bellman} and \eqref{eq:t2diff_bellman}, we have
\begin{align*}
\mathscr{T}_2^\text{diff}
&\leq 2H \sum_{k=1}^{m} \sum_{t=t_k}^{t_{k+1}-1} \max \left\{ \widehat{V}_k(s_t) - P_{s_t,a_t}\widehat{V}_k - \left( \vstar(s_t) - P_{s_t,a_t}\vstar \right), 0 \right\} \\
&\leq 2H \sum_{k=1}^{m} \sum_{t=t_k}^{t_{k+1}-1} \max \left\{ (1-\gamma) P_{s_t,a_t} \left( \vstar - \widehat{V}_k \right) + \gamma\left( \widehat{P}^k_{s_t,a_t} - P_{s_t,a_t} \right)\widehat{V}_k + \gamma b_k\left(s_t,a_t,\widehat{V}_k\right), 0\right\} \\
&\stackrel{(i)}{\leq} 2H \sum_{k=1}^{m} \sum_{t=t_k}^{t_{k+1}-1}\left( \left( P^k_{s_t,a_t} - P_{s_t,a_t} \right)\widehat{V}_k + b_k\left(s_t,a_t,\widehat{V}_k\right) + (1-\gamma)\varepsilon_k\right)\\
&= 2H \tmodel.
\end{align*}
Note that inequality $(i)$ holds under Lemma \ref{lem:optimism}, in which case $(1-\gamma)P_{s_t,a_t}\left( \vstar - \widehat{V}_k \right) \leq (1-\gamma)\max_s\left( \vstar(s) - \widehat{V}_k(s) \right) \leq (1-\gamma)\varepsilon_k$. 

Furthermore, we compute
\begin{align*}
\mathscr{T}_3^\text{diff}
&= \sum_{k=1}^{m}\sum_{t=t_k}^{t_{k+1}-1} \left( \left(\wt{D}_k(s_{t+1})\right)^2 - \left(\wt{D}_k(s_t)\right)^2 \right) \\
&= \sum_{k=1}^{m} \left( \left( \wt{D}_k(s_{t_{k+1}}) \right)^2 - \left( \wt{D}_k(s_{t_k}) \right)^2 \right) \\
&\leq mH^2 \\
&\leq H^2SAU.
\end{align*}

Recombining terms, we have with probability $1 - 3T\delta'$ that
\[
\vardiff \leq \mathscr{T}_1^\text{diff} + \mathscr{T}_2^\text{diff} + \mathscr{T}_3^\text{diff} 
\leq \bigo\left(H\sqrt{U} \sqrt{\vardiff} + H \tmodel + H^2SAU \right) 
,\]
which by Lemma \ref{lem:ab} implies
\[
\vardiff \leq \bigo \left( H \tmodel + H^2SAU \right)
.\]
\end{proof}

\subsection{Proof of Lemma~\ref{lem:tkhat_properties}} \label{sec:tkhat_properties_proof}

\begin{proof}
Let $\mathcal{E}$ be the event that 
\[
	\left| \left(\widehat{P}^k_{s,a} - P_{s,a}\right)\vstar  \right| \leq 2 \sqrt{\frac{\V\left(\widehat{P}^k_{s,a}, \vstar\right)U}{n_k(s,a)}} + \frac{14HU}{3n_k(s,a)} \qquad \forall (s,a,k)\in\St\times\A\times[m].
\]

We restate an alternate, stronger version of Lemma~\ref{lem:tkhat_properties}, which we will prove instead.
Afterwards, we will complete the proof of Lemma~\ref{lem:tkhat_properties} by showing that $\mathcal{E}$ holds with the claimed high probability. 
\begin{lemma} \label{lem:tkhat_appendix}
    The following hold for all $k\in[m]$.
	\begin{enumerate}
		\item $\widehat{\mathcal{T}}_k$ satisfies the constant shift property: for any $Q\in\R^{\St\times\A}$ and $c \in \R$, we have $\widehat{\mathcal{T}}_k (Q + c \mathbf{1}) = \widehat{\mathcal{T}}_k Q + \gamma c \mathbf{1}$. 
		\item $\widehat{\mathcal{T}}_k$ is monotonic: for any $Q,Q'\in\R^{\St\times\A}$ such that $MQ \geq MQ'$, we have $\widehat{\mathcal{T}}_k Q \geq \widehat{\mathcal{T}}_k Q'$. 
		\item $\widehat{\mathcal{T}}_k$ is a $\gamma$-contraction: for any $Q,Q'\in\R^{\St\times\A}$, we have $\left\| \widehat{\mathcal{T}}_k Q - \widehat{\mathcal{T}}_k Q' \right\|_\infty \leq \gamma \|Q - Q'\|_\infty$.    
		\item Under the event $\mathcal{E}$, we also have $\qstar \leq \widehat{\mathcal{T}}_k \qstar$. 
	\end{enumerate}
\end{lemma}

We remark that while we call Part 2 the monotonicity property, it is actually stronger.
Indeed, for $\widehat{\mathcal{T}}_k Q \geq \widehat{\mathcal{T}}_k Q'$ to hold, we only need $MQ \geq MQ'$, which is a weaker requirement than $Q \geq Q'$.
This version will be useful in proving Part 3.
Now, let $k\in[m]$ be arbitrary.

To show Part 1 of Lemma~\ref{lem:tkhat_appendix} (constant shift), let $Q \in \R^{\St\times\A}$ and $c\in\R$ be arbitrary.
Since $\clipH(M(Q + c\mathbf{1})) = \clipH(MQ + c\mathbf{1}) = \clipH(MQ) + c\mathbf{1}$, a straightforward computation gives us
\begin{align*}
	\left(\widehat{\mathcal{T}}_k (Q + c\mathbf{1})\right)(s,a)
	&= r(s,a) + \gamma \widehat{P}^k_{s,a} \left( \clipH(M(Q + c\mathbf{1})) \right) + \gamma b_k\left( s,a, \clipH(M(Q+c\mathbf{1})) \right) \\
	&= r(s,a) + \gamma \widehat{P}^k_{s,a} \left( \clipH(MQ) \right) + \gamma \widehat{P}^k_{s,a} (c\mathbf{1}) + \gamma b_k\left( s,a, \clipH(MQ) \right) \\
	&= \widehat{\mathcal{T}}_k Q + \gamma c\mathbf{1}.
\end{align*}

To show Part 2 of Lemma~\ref{lem:tkhat_appendix} (monotonicity), let $Q, Q' \in \R^{\St\times\A}$ be such that $MQ \geq MQ'$. 
Setting $\alpha \coloneq \min_s (\clipH(MQ'))(s)$ and $\beta \coloneq \min_s \left( (\clipH(MQ))(s) - (\clipH(MQ'))(s) \right)$.
Observe that $\beta > 0, \clipH(MQ')-\alpha\mathbf{1} \in \left[ 0,H \right]^\St, \clipH(MQ) -\alpha\mathbf{1} - \beta\mathbf{1} \in \left[ 0,2H \right]^\St$ and $\clipH(MQ) -\alpha\mathbf{1} - \beta\mathbf{1} \geq \clipH(MQ') -\alpha\mathbf{1}$.
Using $f(p,v,n,u)$ as defined in Lemma~\ref{lem:monotonicity}, we have
\begin{align*}
	&\left(\widehat{\mathcal{T}}_k Q\right)(s,a) \\
	&= r(s,a) + \gamma \widehat{P}^k_{s,a}\left( \clipH(MQ) \right) + \gamma b_k\left( s,a, \clipH(MQ) \right) \\
	&= r(s,a) + \gamma \widehat{P}^k_{s,a} \left( \clipH(MQ) - \alpha\mathbf{1} - \beta\mathbf{1} \right) + \gamma b_k\left(s,a, \clipH(MQ) - \alpha\mathbf{1} - \beta\mathbf{1}\right) + \gamma\alpha + \gamma\beta \\
	&= r(s,a) + \gamma f\left( \widehat{P}^k_{s,a}, \clipH(MQ) - \alpha\mathbf{1} - \beta\mathbf{1}, n_k(s,a), U \right) + \gamma \alpha + \gamma \beta \\
	&\geq r(s,a) + \gamma f\left( \widehat{P}^k_{s,a}, \clipH(MQ') - \alpha\mathbf{1}, n_k(s,a), U \right) + \gamma \alpha + \gamma \beta\\
	&= r(s,a) + \gamma \widehat{P}^k_{s,a} (\clipH(MQ')) + \gamma b_k\left( s,a,\clipH(MQ') \right) + \gamma \beta \\
	&= \left( \widehat{\mathcal{T}}_k Q'\right)(s,a) + \gamma \beta \\
	&\geq \left( \widehat{\mathcal{T}}_k Q' \right)(s,a),
\end{align*}
where the first inequality is due to Lemma~\ref{lem:monotonicity}.

To show Part 3 of Lemma~\ref{lem:tkhat_appendix} ($\gamma$-contraction), let $Q,Q'\in \R^{\St\times\A}$ be arbitrary.
Since $MQ \leq MQ' + \| MQ - MQ' \|_\infty \mathbf{1} = M\left( Q' + \|MQ - MQ' \|_\infty \mathbf{1} \right)$, we have by the monotonicity and constant shift properties of $\widehat{\mathcal{T}}_k$ that
\[
    \widehat{\mathcal{T}}_k Q
	\leq \widehat{\mathcal{T}}_k \left( Q' + \| MQ - MQ' \|_\infty \mathbf{1} \right)
	\leq \widehat{\mathcal{T}}_k Q' + \gamma \| MQ - MQ' \|_\infty
	\leq \widehat{\mathcal{T}}_k Q' + \gamma \| Q - Q' \|_\infty.
\]
Rearranging gives us
\[
    \widehat{\mathcal{T}}_k Q - \widehat{\mathcal{T}}_k Q' \leq \gamma \| Q - Q' \|_\infty.
\]
Reversing the roles of $Q$ and $Q'$ in the above, we also have
\[
    \widehat{\mathcal{T}}_k Q' - \widehat{\mathcal{T}}_k Q \leq \gamma \| Q - Q' \|_\infty,
\]
and combining this with the above, we conclude that
\[
    \left\| \widehat{\mathcal{T}}_k Q - \widehat{\mathcal{T}}_k Q' \right\|_\infty \leq \gamma \| Q - Q' \|_\infty.
\]

Finally, we show Part 4 of Lemma~\ref{lem:tkhat_appendix}.
For any $(s,a)\in\St\times\A$, we have 
\begin{align*}
	\qstar(s,a)
	&= r(s,a) + \gamma P_{s,a} \vstar \\
	&= r(s,a) + \gamma \widehat{P}^k_{s,a} \vstar + \gamma \left( P_{s,a} - \widehat{P}^k_{s,a} \right) \vstar \\
	&\leq r(s,a) + \gamma \widehat{P}^k_{s,a} \clipH (M \qstar) + \gamma b_k\left( s, a,  \vstar \right) \\
	&= \widehat{\mathcal{T}}_k \qstar,
\end{align*}
where the inequality holds under $\mathcal{E}$. 

It remains to show that $\mathcal{E}$ occurs with high probability.
Fix $s\in\St$ and $a\in\A$, and suppose $n_k(s,a)=n$ for some $n\geq 2$.
Note that the $n_k(s,a) = 1$ case trivially always holds. 

Observe that $P^k_{s,a}\vstar = \frac{1}{n} \sum_{i=1}^{n} Z_i$, where $Z_1,\dots,Z_n$ are i.i.d. multinoulli random variables that take the value $\vstar(s')$ with probability  $P_{s,a,s'}$ for each $s'$.
Letting $Z$ be i.i.d. from the same distribution as the  $Z_i$'s, we have that  $\mathbb{E}[Z] = P_{s,a}\vstar$ and in the notation of Lemma~\ref{lem:empirical_bernstein},
\[
	\widehat{\text{Var}}_n = \V\left( \widehat{P}^k_{s,a}, \vstar \right)
.\]
Subsequently, Lemma \ref{lem:empirical_bernstein} gives us that with probability at least $1-\delta'$,
\[
	\left| \widehat{P}^k_{s,a}\vstar - P_{s,a}\vstar \right| 
	\leq 
	\sqrt{\frac{2\V \left( \widehat{P}^k_{s,a}, \vstar \right)\log(2/\delta')}{n-1}} + \frac{7H \log(2/\delta')}{3(n-1)} 
	\leq 
	2\sqrt{\frac{\V\left( \widehat{P}^k_{s,a}, \vstar \right) U}{n}} + \frac{14HU}{3n} 
,\]
where the second inequality follows from the fact that $\frac{1}{x-1} \leq \frac{2}{x}$ for $x\geq 2$, as well as the fact that $\log(2/\delta') \leq U$. 

Taking a union bound over all possible $s,a,n$ thus gives us that $\mathcal{E}$ holds with probability at least $1 - SAT\delta'$.
\end{proof}

\subsection{Proof of Lemma~\ref{lem:wseq}} \label{sec:wseq_proof}

\begin{proof}
We start by proving Part 1, that $\sum_{t=1}^{T} \frac{1}{n_{k_t}(s_t,a_t)} \leq SAU$.
Observe that
\[
	\sum_{t=1}^{T} \frac{1}{n_{k_t}(s_t,a_t)}
	=
	\sum_{t=1}^{T} \sum_{s,a} \frac{\ind\left( (s,a)=(s_t,a_t) \right)}{n_{k_t}(s,a)}
	=
	\sum_{s,a}\sum_{t=1}^{T} \frac{\ind\left( (s,a)=(s_t,a_t) \right)}{n_{k_t}(s,a)}.
\]
Now, fix $(s,a)\in\St\times\A$.
We have
\begin{align*} 
	\sum_{t=1}^{T} \frac{\ind\left( (s,a)=(s_t,a_t) \right)}{n_{k_t}(s,a)}
	&= \sum_{t=1}^{T} \sum_{i=0}^{\imax} \frac{\ind\left( (s,a)=(s_t,a_t),2^i\leq n_{k_t}(s,a)\leq 2^{i+1}-1 \right)}{2^i} \\
	&= \sum_{i=0}^{\imax} \sum_{t=1}^{T} \frac{\ind\left( (s,a)=(s_t,a_t),2^i\leq n_{k_t}(s,a)\leq 2^{i+1}-1 \right)}{2^i} \\
	&\leq \sum_{i=0}^{\imax} \frac{2^i}{2^i} \\
	&= \imax +1 \leq U,
\end{align*}
where the first inequality holds due to the doubling trick.
Namely,
\[
	\sum_{t=1}^{T} \ind\left( (s_t,a_t)=(s,a), 2^i \leq n_{k_t}(s,a) \leq 2^{i+1}-1 \right) \leq 2^i
.\] 
Summing over all possible $(s,a)$ gives us the desired bound. 

We proceed to proving Part 2.
Let $w_1,\dots,w_T \in\R$ be nonnegative numbers. 
\begin{align*}
\sum_{t=1}^{T} \sqrt{\frac{w_t}{n_{k_t}(s_t,a_t)}}
&\leq \sqrt{\sum_{t=1}^{T} \frac{1}{n_{k_t}(s_t,a_t)}} \sqrt{\sum_{t=1}^{T} w_t} \\
&\leq \sqrt{SAU \sum_{t=1}^{T} w_t}.
\end{align*}
The first inequality is by Cauchy-Schwarz, and the second inequality is by Part 1 above.

Lastly, we prove Part 3.
We have
\[
	\sum_{k=1}^{m} \sum_{t=t_k}^{t_{k+1}-1} (1-\gamma)\varepsilon_k
	=
	\sum_{k=1}^{m} \sum_{t=t_k}^{t_{k+1}-1} \frac{1}{t_k}
	=
	\sum_{k=1}^{m} \frac{N_k}{t_k}.
\]
Now, for each $k\in[m]$, it holds that $t_k = 1+\sum_{\ell=1}^{k-1} N_\ell$.
Continuing, we subsequently have
\begin{align*}
    \sum_{k=1}^{m} \frac{N_k}{t_k}
	&= \sum_{k=1}^{m} \frac{N_k}{1 + \sum_{\ell=1}^{k-1} N_\ell } \\
	&\leq \sum_{k=1}^{m }\log\left( 1 + \frac{N_k}{1+\sum_{\ell=1}^{k-1}N_\ell}\right) \\
	&= \sum_{k=1}^{m} \log\left( \frac{1 + \sum_{\ell=1}^{k} N_\ell}{1 + \sum_{\ell=1}^{k-1} N_\ell} \right) \\
	&= \sum_{k=1}^{m} \left( \log\left( 1 + \sum_{\ell=1}^{k} N_\ell \right) - \log\left( 1 + \sum_{\ell=1}^{k-1} N_\ell \right) \right) \\
	&= \log\left( 1 + \sum_{\ell=1}^{m} N_\ell \right) - \log( 1 ) \\
	&= \log\left( 1 + T \right) \\
	&\leq U.
\end{align*}
\end{proof}

\section{Proof of Corollary \ref{cor:regret_no_prior}} \label{sec:regret_no_prior_proof}
Here we prove Corollary \ref{cor:regret_no_prior}. 

\begin{proof}

    After applying Theorem \ref{thm:var_dependent_reg} with span bound $H=\sqrt{\frac{T}{S^3A}}$ and combining with Lemma \ref{lem:var_star_bound} and the fact that $\spannorm{\vstar_\gamma}\leq  2\spannorm{\hstar} $ \citep{wei_model-free_2020} to bound the variance parameter (and taking a union bound to get a failure probability of at most $2\delta$ that these do not both hold) all that remains is to confirm that the resulting regret bound of
    \begin{align}
        \reg(T) \leq \tbigo\Big( \sqrt{\left(\spannorm{\hstar} T + \spannorm{\hstar}^2 \right)SA} + \sqrt{SAT}  \Big) \label{eq:cor_7_pf_1}
    \end{align}
    whenever $T \geq \spannorm{\hstar}^2 S^3 A$
    implies the claim that $\reg(T) \leq \tbigo\left(\sqrt{(\tspannorm{\hstar}+1)SAT}\right)$ for $T \geq \spannorm{\hstar}^2 S^3 A$.

    Observe that $T \geq \spannorm{\hstar}^2 S^3 A \geq \spannorm{\hstar}^2$ implies that $\spannorm{\hstar}\sqrt{SA} \leq \sqrt{SAT} $, so~\eqref{eq:cor_7_pf_1} simplifies to a bound of
    \begin{align*}
        \reg(T) \leq \tbigo\Big( \sqrt{\spannorm{\hstar} SAT } + \sqrt{SAT}  \Big)
        =
        \tbigo\Big( \sqrt{\left(\spannorm{\hstar} + 1\right) SAT }  \Big) 
    \end{align*}
    whenever $T \geq \spannorm{\hstar}^2 S^3 A$, as required.
    
    Note that the second regret bound stated in Corollary~\ref{cor:regret_no_prior} follows by considering cases.
    Indeed, for $T\leq \spannorm{\hstar}^2 S^3 A$, the regret can be bounded by $T$, while for $T \geq \spannorm{\hstar}^2 S^3 A$ we have just shown that the regret is bounded by $\tbigo
    \Big(\sqrt{(\tspannorm{\hstar}+1)SAT}\Big)$.
    So for all $T$, the regret is bounded by $\tbigo\Big(\sqrt{(\tspannorm{\hstar}+1)SAT} + \spannorm{\hstar}^2 S^3 A$\Big).
\end{proof}

\section{Proof of Theorem~\ref{thm:burn_in_lb}} \label{sec:burn_in_lb_proof}

\begin{proof}
Let $S \geq 2$ and $A \geq 2$ be integers, and let $D \geq 4\left\lceil \log_A S \right\rceil$.
Suppose that $T \leq \frac{1}{32}DSA$.

Let $S' = \lceil \frac{S-1}{2} \rceil$ and $A' = A-1$, and observe that $T \leq \frac{1}{8}DS'A'$ and $2\left\lceil \log_{A}S \right\rceil \leq \frac{D}{2}$ (we will need these facts later). 
We construct a family of hard MDPs $\{P_{(s,a)} \mid (s,a)\in[S']\times[A'] \}$ such that each MDP has $S$ states, $A$ actions, and diameter at most $D$ (see Figure \ref{fig:lb_construction}).
First, we use an $A$-ary tree structure to connect states $1,\dots,S-1$ by assigning $A$ actions with deterministic transitions at each non-leaf state.
In particular, each action goes to a distinct child node.
If a non-leaf node has fewer than $A$ children, the remaining actions are deterministic self-loops.
Note that we can (and do) choose a tree with depth at most $\lceil \log_A S \rceil$ and at least $S'$ leaf nodes.
Next, we let state $S$ be the ``good'' state, where all actions are deterministic self-loops with reward 1, except for one action which deterministically returns to the root of the tree.
At every other node the reward for any action is 0, so the learner's goal is to reach state $S$.

To reach state $S$, the learner must search for the correct action in the correct leaf node, which is different in each MDP instance.
In the MDP $P_{(s,a)}$, the correct state-action pair is $(s,a)$.
That is, at all leaf nodes except state $s$, $A'$ actions are deterministic self-loops, and the remaining action is a deterministic transition back to the root.
State $s$ is identical except that action $a$ transits to the good state with probability $\frac{2}{D}$, and stays in its current state with probability $1-\frac{2}{D}$.
Note that the diameter of this MDP is $2\left\lceil \log_{A}S \right\rceil + \frac{D}{2} \leq D$ as required.

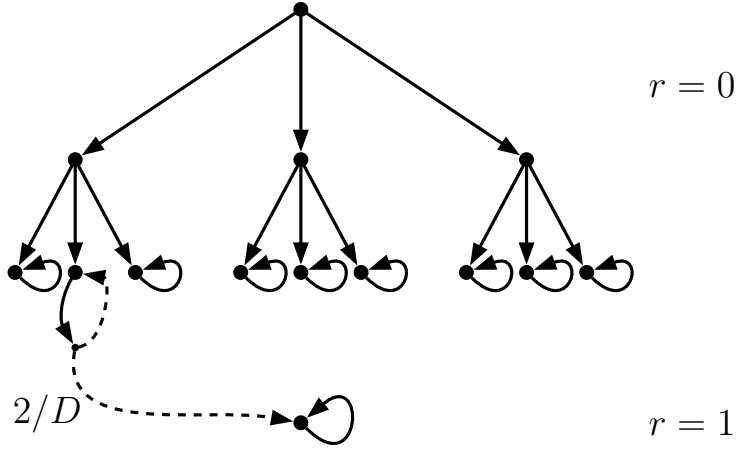
\begin{figure}
\centering
\begin{minipage}{1\textwidth}
\centering
\begin{tikzpicture}[
    scale=1,
    node distance=1.5cm,
    dot/.style={circle, fill=black, inner sep=2pt, outer sep=0pt},
    edge/.style={draw=black, line width=1.2pt, ->, -{Triangle[length=3mm, width=2mm]}, line cap=round},
    loop edge/.style={draw=black, line width=1.2pt, ->, -{Triangle[length=3mm, width=2mm]}, looseness=20, min distance=8mm},
    dashed edge/.style={draw=black, dashed, line width=1.2pt, ->, -{Triangle[length=3mm, width=2mm]}},
    font=\sffamily\bfseries\Large
]

    
    \node[dot] (root) at (0,0) {};
    
    \node[dot] (L1) at (-3, -2) {};
    \node[dot] (L2) at (0, -2) {};
    \node[dot] (L3) at (3, -2) {};
    
    \node[dot] (L1C1) at (-3.8, -3.5) {};
    \node[dot] (L1C2) at (-3.0, -3.5) {}; 
    \node[dot] (L1C3) at (-2.2, -3.5) {};

    \node[dot, inner sep=1pt] (action) at (-3.0, -4.5) {};
    
    \node[dot] (L2C1) at (-0.8, -3.5) {};
    \node[dot] (L2C2) at (0, -3.5) {};
    \node[dot] (L2C3) at (0.8, -3.5) {};
    
    \node[dot] (L3C1) at (2.2, -3.5) {};
    \node[dot] (L3C2) at (3.0, -3.5) {};
    \node[dot] (L3C3) at (3.8, -3.5) {};
    
    \node[dot] (R1) at (0, -5.5) {};

    \draw[edge] (root) -- (L1);
    \draw[edge] (root) -- (L2);
    \draw[edge] (root) -- (L3);
    
    \draw[edge] (L1) -- (L1C1);
    \draw[edge] (L1) -- (L1C2);
    \draw[edge] (L1) -- (L1C3);
    
    \draw[edge] (L2) -- (L2C1);
    \draw[edge] (L2) -- (L2C2);
    \draw[edge] (L2) -- (L2C3);
    
    \draw[edge] (L3) -- (L3C1);
    \draw[edge] (L3) -- (L3C2);
    \draw[edge] (L3) -- (L3C3);

    \foreach \n in {L1C1, L1C3, L2C1, L2C2, L2C3, L3C1, L3C2, L3C3} {
        \draw[loop edge, looseness=18, min distance=5mm] (\n) to [out=315, in=20] (\n);
    }
    
    \draw[loop edge] (R1) to [out=315, in=40] (R1);

    \draw[edge, draw=black] (L1C2) to [out=-120, in=120] (action);
    \draw[dashed edge] (action) to [out=10, in=-10] (L1C2);
    \draw[dashed edge] (action) to [out=-100, in=170] 
        node[midway, left=7mm, black] {$2/D$} (R1);
     
    \node[anchor=west] at (4.5, -1) {$r=0$};
    \node[anchor=west] at (4.5, -5.5) {$r=1$};

\end{tikzpicture}
\end{minipage}

\caption{An example of a hard MDP construction for $S=14$ and $A=3$. 
To avoid clutter, we omit an additional deterministic self-loop at each leaf state.
We also omit the deterministic actions which transit from the leaf states to the root and from the good state to the root, as these actions only serve to keep the diameter bounded by $D$.}
\label{fig:lb_construction}
\end{figure}

Now fix any horizon-$T$ algorithm.
For any $\theta\in[S']\times[A']$, write
$\mathbb{E}_\theta$ to denote the expectation induced by the algorithm when the underlying MDP is $P_\theta$, and write $\P_a$ for the corresponding probability.
For each $(s,a)$, let the random variable $N_{(s,a)}$ be the number of times action $a$ is taken in state $s$, and let $N$ be the total number of visits to the bad states with reward 0.
We observe that $\reg(T) \geq N$, so our goal is to lower bound $\E_\theta[N]$ for some $\theta$.

Let $\mathcal{E}$ denote the event that the learner does \textbf{not} observe a transition to the good state.
Further let $\theta\in[S']\times[A']$.
We start with the simple observation that since,
\[
	T 
    \geq 
    \E_\theta[N \mid \mathcal{E}] 
    \geq 
    \sum_{(s,a)\in[S']\times[A']} \E_{\theta}[N_{(s,a)} \mid \mathcal{E}],
\]
it follows that for some $(s',a')$ we have $\E_\theta[N_{(s',a')} \mid \mathcal{E}] \leq \frac{T}{S'A'} \leq \frac{D}{8}$.
We then claim that 
\[
\E_{(s',a')}[N_{(s',a')}]
\leq
\E_{(s',a')}[N_{(s',a')} \mid \mathcal{E}]
=
\E_{\theta}[N_{(s',a')} \mid \mathcal{E}] 
\leq 
\frac{D}{8}.
\]
The first inequality holds because we can assume WLOG that the algorithm will stay in the good state upon transiting there.
The second inequality holds because the algorithm will behave exactly the same on any of the hard MDPs under the event $\mathcal{E}$.
So, denoting $\theta' \coloneq (s',a')$, we have shown that $\E_{\theta'}[N_{\theta'}] \leq \frac{D}{8}$.

Next, we compute that
\[
\P_{\theta'}(\mathcal{E})
\geq
\P_{\theta'}\left(\mathcal{E} \mid N_{\theta'} < D/4\right) \P_{\theta'}\left(N_{\theta'} < D/4\right)
\geq
\left(1-\frac{2}{D}\right)^{D/4} \frac{1}{2}
\geq
\frac{1}{4}
\]
where the last inequality is due to Bernoulli's inequality (Lemma~\ref{lem:bernoulli}), and the second inequality is due to
\[
\P_{\theta'}(N_{\theta'} \geq D/4) \leq \frac{\E_{\theta'}[N_{\theta'}]}{D/4} \leq \frac{1}{2}.
\]
Since
\[
\E_{\theta'}[N]
\geq
\E_{\theta'}[N\mid \mathcal{E}] \P_{\theta'}(\mathcal{E})
\geq
\frac{T}{4},
\]
we conclude that $\E_{\theta'}[\reg(T)] \geq \frac{T}{4}$.
\end{proof}

\section{Proof of Theorem~\ref{thm:no_prior_knowledge_H2_lb}} \label{sec:no_prior_H2_lb_proof}

In this section we prove the lower bound Theorem \ref{thm:no_prior_knowledge_H2_lb}. First we show the following intermediate result.

\begin{figure}[ht]
    \centering
    \begin{minipage}{0.45\textwidth}
        \centering
        \begin{tikzpicture}[
    scale=0.8,
    node distance=1.5cm,
    dot/.style={circle, fill=black, inner sep=2pt, outer sep=0pt},
    edge/.style={draw=black, line width=1pt, ->, -{Triangle[length=3mm, width=2mm]}, line cap=round},
    loop edge/.style={draw=black, line width=1pt, ->, -{Triangle[length=3mm, width=2mm]}, looseness=20, min distance=8mm},
    dashed edge/.style={draw=black, dashed, line width=1pt, ->, -{Triangle[length=3mm, width=2mm]}},
    font=\sffamily\bfseries
]

    
    \node[dot] (root) at (0,0) {};
    
    \node[dot] (L1) at (-2, -2) {};
    \node[dot] (L2) at (2, -2) {};
    
    \node[dot] (L1C1) at (-2.7, -3.5) {};
    \node[dot] (L1C2) at (-1.3, -3.5) {}; 

    \node[dot, inner sep=1pt] (action) at (-1.3, -4.5) {};
    
    \node[dot] (L2C1) at (1.3, -3.5) {};
    \node[dot] (L2C2) at (2.7, -3.5) {};

    \node[dot] (R1) at (0, -5.5) {};

    \draw[edge] (root) -- (L1);
    \draw[edge] (root) -- (L2);
    
    \draw[edge] (L1) -- (L1C1);
    \draw[edge] (L1) -- (L1C2);
    
    \draw[edge] (L2) -- (L2C1);
    \draw[edge] (L2) -- (L2C2);

    \foreach \n in {L1C1, L2C1, L2C2} {
        \draw[loop edge, looseness=18, min distance=5mm] (\n) to [out=315, in=20] (\n);
    }
    
    \draw[loop edge] (R1) to [out=315, in=40] (R1);

    \draw[edge, draw=black] (L1C2) to [out=-120, in=120] (action);
    \draw[dashed edge] (action) to [out=10, in=-10] (L1C2);
    \draw[dashed edge] (action) to [out=-100, in=170] 
        node[midway, left=2mm, black] {$2/B$} (R1);
     
    \node[anchor=west] at (0.5, 0) {$r=1/2$};
    \node[anchor=west] at (2.5, -2) {$r=1/2$};
    \node[anchor=west] at (4, -3.5) {$r=0$};
    \node[anchor=west] at (1, -5.5) {$r=1$};
    
        \end{tikzpicture}
        \\[1ex] 
        $P_{(1,2,1)}$
    \end{minipage}
    \hfill 
    \begin{minipage}{0.45\textwidth}
        \centering
        \begin{tikzpicture}[
    scale=0.8,
    node distance=1.5cm,
    dot/.style={circle, fill=black, inner sep=2pt, outer sep=0pt},
    edge/.style={draw=black, line width=1pt, ->, -{Triangle[length=3mm, width=2mm]}, line cap=round},
    loop edge/.style={draw=black, line width=1pt, ->, -{Triangle[length=3mm, width=2mm]}, looseness=20, min distance=8mm},
    dashed edge/.style={draw=black, dashed, line width=1pt, ->, -{Triangle[length=3mm, width=2mm]}},
    font=\sffamily\bfseries
]

    
    \node[dot] (root) at (0,0) {};
    
    \node[dot] (L1) at (-2, -2) {};
    \node[dot] (L2) at (2, -2) {};
    
    \node[dot] (L1C1) at (-2.7, -3.5) {};
    \node[dot] (L1C2) at (-1.3, -3.5) {}; 

    \node[dot, inner sep=1pt] (action) at (-1.3, -4.5) {};
    
    \node[dot] (L2C1) at (1.3, -3.5) {};
    \node[dot] (L2C2) at (2.7, -3.5) {};

    \node[dot] (R1) at (0, -5.5) {};

    \draw[edge] (root) -- (L1);
    \draw[edge] (root) -- (L2);
    
    \draw[edge] (L1) -- (L1C1);
    \draw[edge] (L1) -- (L1C2);
    
    \draw[edge] (L2) -- (L2C1);
    \draw[edge] (L2) -- (L2C2);

    \foreach \n in {L1C1, L2C1, L2C2} {
        \draw[loop edge, looseness=18, min distance=5mm] (\n) to [out=315, in=20] (\n);
    }
    
    \draw[edge] (R1) to[out=60, in=-100] (root);
    

    \draw[edge, draw=black] (L1C2) to [out=-120, in=120] (action);
    \draw[dashed edge] (action) to [out=10, in=-10] (L1C2);
    \draw[dashed edge] (action) to [out=-100, in=170] 
        node[midway, left=2mm, black] {$2/B$} (R1);
     
    \node[anchor=west] at (0.5, 0) {$r=1/2$};
    \node[anchor=west] at (2.5, -2) {$r=1/2$};
    \node[anchor=west] at (4, -3.5) {$r=0$};
    \node[anchor=west] at (1, -5.5) {$r=1$};
    
        \end{tikzpicture}
        \\[1ex] 
        $P_{(2,2,1)}$
    \end{minipage}
    \caption{An example of the MDPs used in the proof of Theorem~\ref{thm:lower_bound_prior}. If the transition associated with a state-action pair is deterministic, it is denoted with a solid arrow. If it is stochastic, it is represented as a solid line splitting into multiple dashed arrows to different states, each annotated with the associated probability of that transition. The MDPs are parameterized by $B > 1.$
    Some actions, such as those which transit from leaf states back to the root state, are omitted.}
    \label{fig:lower_bound_prior_mdp}
\end{figure}
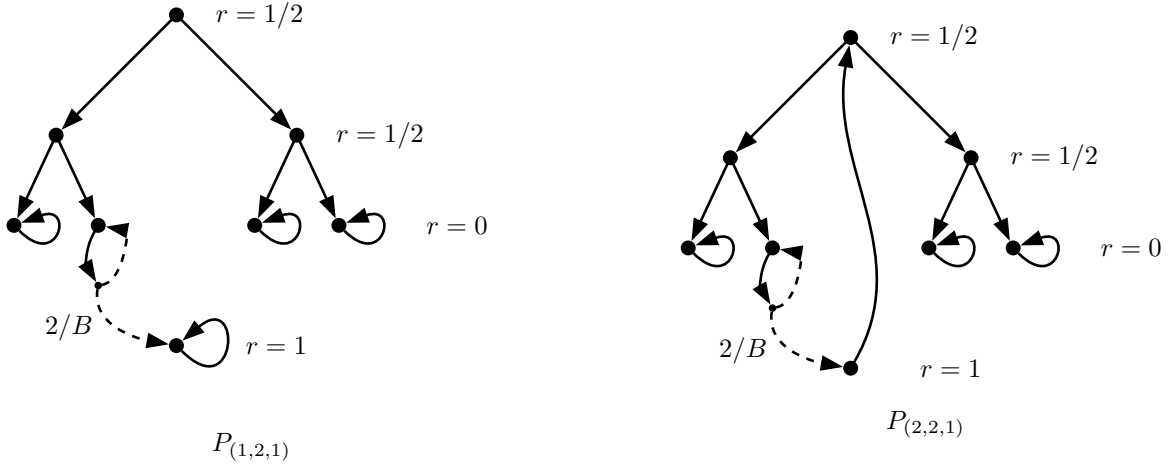

\begin{theorem} \label{thm:lower_bound_prior}
    There exist universal constants $c_1>1$ and $0<c_2<1$ such that the following holds.
    Fix integers $T \geq 1, S\geq 2, A\geq 2$, and fix $B \geq \max\{c_1, 2\lceil\log_A S\rceil\}$. Let $\texttt{Alg}$ be any horizon-$T$ algorithm. Then there exist two communicating MDPs $P_1, P_2$ such that:
    \begin{enumerate}
        \item $P_1$ and $P_2$ both have $S$ states and $A$ actions.
        \item $\spannorm{\hstar_{P_1}} = B$ and $\spannorm{\hstar_{P_2}} = \frac{1}{2}$.
        \item If $\E_{P_1}^\texttt{Alg} [\reg(T)] < T/4$, then $\E_{P_2}^\texttt{Alg} [\reg(T)] \geq c_2 BSA$.
    \end{enumerate}
\end{theorem}

\begin{proof}
Let $T\geq 1, S\geq 2, A\geq 2$ be integers, and let $B\geq \max\{50, \lceil \log_A S\rceil\}$.
Further define $S' = \lceil \frac{S-1}{2}\rceil$ and $A'=A-1$.
We will construct a family of MDPs $\left\{P_{(i,s,a)} \mid (i,s,a)\in [2]\times[S']\times [A']\right\}$, where each $P_{(i,s,a)}$ will have a tree construction similar to that used in the proof of Theorem~\ref{thm:burn_in_lb} in Section~\ref{sec:burn_in_lb_proof} (see Figure~\ref{fig:lower_bound_prior_mdp}). 
Specifically, we again use an $A$-ary tree structure to connect states $1,\dots,S-1$ by assigning $A$ actions with deterministic transitions at each non-leaf state, with each action going to a distinct child node.
If a non-leaf node has fewer than $A$ children, the remaining actions are deterministic self-loops.
We use a tree with depth at most $\lceil \log_A S \rceil \leq B/2$ and at least $S'$ leaf nodes.

For the MDPs $P_{(1,s,a)}$ and $P_{(2,s,a)}$, in all leaf states other than $s$, actions $1,\dots,A'$ are deterministic self-loops, and action $A$ is a deterministic transition back to the root.
State $s$ is identical except that action $a$ transits to state $S$ outside the tree with probability $2/B$.
In state $S$, actions $2,\dots,A$ are all deterministic transitions back to the root.
The reward is 1/2 for any action in any non-leaf state or any deterministic transition from a tree state back to the root.
The reward for action 1 in state $S$ is 1, and the reward at any other state-action pair is 0.

The only difference between $P_{(1,s,a)}$ and $P_{(2,s,a)}$ is that in $P_{(1,s,a)}$, action 1 is a deterministic self-loop, while in $P_{(2,s,a)}$, action 1 is a deterministic transition back to the root.
Subsequently, it is straightforward to see that the optimal strategy in $P_{(1,s,a)}$ is to take action $(s,a)$ until a transition to state $S$ is observed, then repeatedly take action 1.
On the other hand, in $P_{(2,s,a)}$ it is not worth trying to reach state $S$, so the optimal strategy is to stay in the tree and only take reward 1/2 actions.
For a learner to distinguish between the two MDPs, it must reach state $S$ and take action 1, but that will take many time steps when the correct $(s,a)$ is not known beforehand.
One can easily verify that the span of the optimal bias function is at most $B$ for all $P_{(1,s,a)}$ and $\frac{1}{2}$ for all $P_{(2,s,a)}$, as required.

We now consider any horizon-$T$ algorithm \texttt{Alg}.
WLOG we can assume the following:
\begin{enumerate}[noitemsep,nolistsep]
    \item \texttt{Alg} is deterministic, meaning that given a sequence of past states and actions, the next action is computed by a deterministic function.
    This assumption can be justified via a standard argument that any randomized strategy is equivalent to some random choice from the set of all deterministic strategies (\cite{auer2002nonstochastic}, \cite{auer_near-optimal_2008}).
    \item Once a transition to state $S$ is observed, \texttt{Alg} acts optimally.
    If this were not the case, the expected regret would only increase.
\end{enumerate}
Since we assume that \texttt{Alg} is deterministic, under the event that no transition to state $S$ occurs \texttt{Alg}, will always observe the deterministic sequence of state-action pairs $(s^{(1)},a^{(1)}),\dots,(s^{(T)},a^{(T)})$. For any $(s,a)$, we denote by $t_k(s,a)$ the index of the $k$th occurrence of $(s,a)$ in this sequence, and we denote by $n_t(s,a)$ the number of times $(s,a)$ occurs through index $t$ in this sequence. 

For $\theta\in [2]\times[S']\times[A']$, let $\P_\theta$ and $\E_\theta$ denote the probability and expectation, respectively, induced by \texttt{Alg} when the underlying MDP is $P_\theta$.
For each $(s,a)$, let $N_{(s,a)}$ denote the number of times the learner takes action $a$ in state $s$.
Further write $\nleave = \sum_{(s,a)\in[S']\times[A']} N_{(s,a)}$, and let $\nstay$ be the number of times the learner takes a reward $\frac{1}{2}$ action.
Additionally, let $\mathcal{E}$ denote the event that the algorithm observes a transition to state $S$.

When the underlying MDP is some $P_{(1,s,a)}$ the optimal gain is 1, and hence the algorithm adds 1 to the regret any time it tries to reach state $S$ (i.e. some state-action in $[S']\times[A']$ and adds 1/2 to the regret any time it takes a reward 1/2 action within the tree. In particular,
\[
\reg(T, P_{(1,s,a)},\texttt{Alg}) \geq T - \frac{1}{2} \nstay - N_{(S,1)}.
\]
Similarly, when the underlying MDP is some $P_{(2,s,a)}$, the optimal gain is 1/2, and hence the algorithm adds 1/2 to the regret any time it tries to reach state $S$. It may subtract 1/2 from the regret in the event that it does $(S,1)$. Consequently,
\[
\reg(T, P_{(2,s,a)},\texttt{Alg}) \geq \frac{1}{2}\nleave - \frac{1}{2}N_{(S,1)} \geq \frac{1}{2}\nleave - \frac{1}{2}
,\]
with the second inequality due to the fact that the algorithm will try $(S,1)$ at most once in $P_{(2,s,a)}$. 

Now, let $(s,a)\in[S']\times[A']$.
Assuming that $\E_{(1,s,a)}[\reg(T)] < \frac{T}{4}$, we will work to show a lower bound on $n_T(s,a)$.
Under our assumption, we have
\begin{align*}
&\frac{T}{4} > T - \frac{1}{2}\E_{(1,s,a)}\left[\nstay \right] - \E_{(1,s,a)}\left[N_{(S,1)}\right] \\
& \implies 
\frac{1}{2}\E_{(1,s,a)}[\nstay] + \E_{(1,s,a)}[N_{(S,1)}] > \frac{3T}{4} \\
& \implies
\E_{(1,s,a)}[N_{(S,1)}] > \frac{T}{4},
\end{align*}
with the second implication following from the trivial fact that $\nstay \leq T$.
Furthermore, since $N_{(S,1)} \leq T\ind(\mathcal{E})$, we have
\begin{equation*}
\frac{T}{4} < \E_{(1,s,a)}[N_{(S,1)}] \leq T\P_{(1,s,a)}(\mathcal{E})
\implies
\P_{(1,s,a)}(\mathcal{E}) > \frac{1}{4}.
\end{equation*}
We have shown that there is a constant probability of observing a transition to state $S$, and we will use this fact to show that $N_{(s,a)}$ is large with constant probability.
Towards this end, we compute
\begin{align*}
\P_{(1,s,a)}\left(\mathcal{E}^c\right)
&\geq  \P_{(1,s,a)}\left(\mathcal{E}^c \mid N_{(s,a)} < \frac{B}{10} \right) \P_{(1,s,a)}\left(N_{(s,a)} < \frac{B}{10} \right) \\
&\geq \left( 1-\frac{2}{B} \right)^\frac{B}{10} \P_{(1,s,a)}\left(N_{(s,a)} < \frac{B}{10} \right) \\
&\geq \frac{4}{5} \P_{(1,s,a)}\left(N_{(s,a)} < \frac{B}{10} \right),
\end{align*}
with the last inequality holding due to Bernoulli's inequality (Lemma~\ref{lem:bernoulli}).
Hence,
\[
\frac{1}{4}
<
\P_{(1,s,a)}(\mathcal{E})
=
1 - \P_{(1,s,a)}\left(\mathcal{E}^c\right)
\leq
1- \frac{4}{5} \P_{(1,s,a)}\left(N_{(s,a)} < \frac{B}{10} \right),
\]
which implies that
\[
\P_{(1,s,a)}\left(N_{(s,a)} < \frac{B}{10} \right) < \frac{3}{4} \cdot \frac{5}{4} = \frac{15}{16}.
\]
Therefore,
$\P_{(1,s,a)}(N_{(s,a)} \geq \frac{B}{10} ) \geq \frac{1}{16}$, which implies that $n_T(s,a) \geq \frac{B}{10}$ (otherwise $\P_{(1,s,a)}(N_{(s,a)} \geq \frac{B}{10})$ would be 0). 
Since $(s,a)$ was arbitrary, we have shown that $n_T(s,a) \geq \frac{B}{10}$ for all $(s,a)\in[S']\times[A']$.

Our next step is to derive a lower bound on $\E_{(1,s,a)}[\nleave]$ for some $(s,a)$.
Observe that for any $(s,a)$, we have
\[
\E_{(1,s,a)}[\nleave]
=
\sum_{t=1}^T t \,\P_{(1,s,a)}(\nleave = t).
\]
Furthermore, for $t<T$, $\nleave=t$ occurs precisely when $(s_t,a_t)=(s,a)$, all previous occurrences of $(s,a)$ do not result in a transition to $S$, and this occurrence of $(s,a)$ does result in a transition to $S$.
$\nleave=T$ is similar except that it does not require a transition to $S.$
In other words, we have
\[
\E_{(1,s,a)}[\nleave]
\geq
\sum_{t=1}^T t \, \ind\Big((s_t,a_t)=(s,a)\Big) \left(1 - \frac{2}{B}\right)^{n_t(s,a)-1} \frac{2}{B}.
\]
Writing $t_k(s,a)$ to be the $k$th time $(s,a)$ occurs, we can lower bound this sum by
\[
\sum_{k=1}^{\lfloor B/8 \rfloor} t_k(s,a) \left(1 - \frac{2}{B}\right)^{k-1} \frac{2}{B}.
\]
The following technical lemma, the proof of which we postpone, shows that we can bound this sum for at least one $(s,a)$.
Intuitively, the sum will be large enough when the $t_k(s,a)$ are large, and because all $(s,a)$ occur many times, there is at least one $(s,a)$ whose indices are sufficiently large.
\begin{lemma} \label{lem:sum_lower_bound}
    Let $B>0$ and $c\in(0,1)$ such that $\lceil cB \rceil \geq 2$.
    Let $M$ be a positive integer.
    Let $z_1,z_2,\dots$ be a sequence of integers such that each $i\in\{1,\dots,M\}$ occurs at least $cB$ times.
    Let $t_k(i)$ be the index of the $k$th occurrence of value $i$.
    Then
    \[
    \max_{i\in\{1,\dots,M\}} \sum_{k=1}^{\lfloor cB \rfloor} t_k(i) \left(1-\frac{2}{B}\right)^{k-1} \frac{2}{B} \geq \frac{c^2(1-2c)}{2} BM.
    \]
    
\end{lemma}
An application of Lemma~\ref{lem:sum_lower_bound} with $c = 1/10$ and $M=S'A'$ gives us that there exists some $(s',a')$ satisfying
\[
\E_{(1,s',a')}[\nleave] \geq \frac{1}{250} BS'A'.
\]

Finally, it is not hard to see that $\E_{(1,s',a')}[\nleave] = \E_{(2,s',a')}[\nleave]$, so
\[
\E_{(2,s',a')}[\reg(T,P_{(2,s,a)},\texttt{Alg})] \geq \E_{(2,s',a')}\left[ \frac{1}{2} \nleave - \frac{1}{2}\right] 
\geq
\frac{BS'A'}{250} - \frac{1}{2} 
\geq
\frac{BSA}{2000} - \frac{1}{2}
\geq
\frac{BSA}{4000},
\]
where the final inequality is due to the fact that $B\geq 500 \implies \frac{BSA}{4000} \geq \frac{1}{2}$.
We conclude that the requirements of the theorem hold with $c_1 = 500, c_2=1/4000, P_1 = P_{(1,s',a')}$, and $P_2 = P_{(2,s',a')}$.

\end{proof}

We now prove Lemma~\ref{lem:sum_lower_bound}.
\begin{proof}
Fix arbitrary $B>0$ and $c\in(0,1)$ satisfying $\lceil cB \rceil \geq 2$, let $M\in\mathbb{Z}_{\geq 1}$, and let $z_1,z_2,\dots$ be a sequence of integers such that $\sum_{j=1}^\infty \ind(z_j = i) \geq cB$ for all $i\in\{1,\dots,M\}$.
For ease of presentation, write $B'\coloneq \lceil cB \rceil$ and $w_k \coloneq \left( 1 - \frac{2}{B} \right)^{k-1} \frac{2}{B}$ so that our goal is to bound
\[
\max_{i\in\{1,\dots,M\}} \sum_{k=1}^{B' } t_k(i) w_k.
\]
First, we use that the max is greater the average, so that
\[
\max_{i\in\{1,\dots,M\}} \sum_{k=1}^{B'} t_k(i) w_k
\geq
\frac{1}{M}\sum_{i=1}^M\sum_{k=1}^{B'} t_k(i) w_k.
\]
Observing that $\{t_k(i)\}=\{1,\dots,MB'\}$, we then reindex the sum and write
\[
\frac{1}{M}\sum_{i=1}^M\sum_{k=1}^{B'} t_k(i) w_k
=
\frac{1}{M}\sum_{t=1}^{MB'} t\, w_{n(t)},
\]
where $n(t)$ is defined as the number of times that $z_t$ appears through index $t$.
Furthermore, the rearrangement inequality (Lemma~\ref{lem:rearrangement}) gives us that
\[
\frac{1}{M}\sum_{t=1}^{MB'} t\, w_{n(t)}
\geq
\frac{1}{M}\sum_{t=1}^{MB'} t\, w_{\lceil t/M \rceil}
\geq
\frac{1}{M}\sum_{j=1}^{B'} M((j-1)M+1) w_j
\geq
M\sum_{j=1}^{B'} (j-1) w_j.
\]
It remains to analyze $M\sum_j (j-1)w_j$.
Note that for any $j\in\{1,\dots,B'\}$, Bernoulli's inequality (Lemma~\ref{lem:bernoulli}) gives us
\[
w_j 
= 
\left(1-\frac{2}{B}\right)^{j-1} \frac{2}{B}
\geq
\left(1-\frac{2}{B}\right)^{B'-1} \frac{2}{B}
\geq
\left(1-\frac{2(B'-1)}{B}\right) \frac{2}{B}
\geq
\frac{2-4c}{B}.
\]
Consequently,
\[
M\sum_{j=1}^{B'} (j-1)w_j
\geq
\frac{M(2-4c)}{B} \sum_{j=1}^{B'-1} j
=
\frac{M(2-4c)}{B} \frac{B'(B'-1)}{2}
\geq
M \frac{1}{B} \frac{2-4c}{2} cB \frac{cB}{2}
\geq
\frac{c^2(1-2c)}{2} BM.
\]
\end{proof}

Since we have fully established Theorem~\ref{thm:lower_bound_prior}, we can turn to proving Theorem~\ref{thm:no_prior_knowledge_H2_lb}.
\begin{proof}
Let $S\geq 2$ and $A\geq 2$ be integers.
Fix some $\alpha\in[1,2)$, and let $T > SA(c_4 \beta_T)^{\frac{4}{2-\alpha}}$, where $c_4\geq 1$ is a universal constant to be defined later.
Suppose that some horizon-$T$ algorithm \texttt{Alg} has for all MDPs $P$,
\begin{align*}
    \E[\reg(T,P,\texttt{Alg})] \leq \sqrt{\beta_T \spannorm{h_{P}^\star} SAT} + \beta_T SA \spannorm{h_{P}^\star}^{\alpha}.
\end{align*}
We want to use Theorem \ref{thm:lower_bound_prior} to show a contradiction. Hence we need a choice of $B$ such that
\begin{align}
    \sqrt{\beta_T B SAT} + \beta_T SA B^\alpha  &< T/4 \label{eq:lb_cor_cond_1}\\
    \sqrt{\beta_T SA T/2} + \beta_T SA/2^\alpha & < c_2 BSA. \label{eq:lb_cor_cond_2}
\end{align}
We will set $B$ in terms of $T,S,A$ to be as small as possible such that the second inequality holds, and then show that the first inequality holds under the assumed conditions.

Assuming that $T > \beta_T SA$, we can then derive that 
\begin{align*}
    \sqrt{\beta_T SA T/2} + \beta_T SA/2^\alpha & < 2\sqrt{\beta_T SA T/2},
\end{align*}
so we can satisfy~\eqref{eq:lb_cor_cond_2} by setting $c_2 BSA = 2\sqrt{\beta_T SA T/2} \iff B = \sqrt{\frac{c_3 \beta_T T}{SA}}$, where $c_3\coloneq \frac{2}{c_2\sqrt{2}}$. Now checking that our choice of $B$ admits~\eqref{eq:lb_cor_cond_1}, we calculate the equivalence
\begin{align*}
    &\sqrt{\beta_T B SA T} + \beta_T SA B^\alpha < T/4 \\
    \iff & \left(c_3^2 \beta_T^3 T^3 SA\right)^\frac{1}{4} + \beta_T SA \left(\sqrt{\frac{c_3 \beta_T T}{SA}} \right)^\alpha < T/4. 
\end{align*}
Hence a sufficient condition for~\eqref{eq:lb_cor_cond_1} is that both of the following are true:
\begin{align}
    \left(c_3^2 \beta_T^3 T^3 SA\right)^\frac{1}{4} & < T/8 \label{eq:lb_cor_cond_3}\\
    \beta_T SA \left( \sqrt{\frac{c_3 \beta_T T}{SA}} \right)^\alpha & < T/8. \label{eq:lb_cor_cond_4}
\end{align}
The condition~\eqref{eq:lb_cor_cond_3} is equivalent to 
\begin{align*}
    & \left(c_3^2 \beta_T^3 T^3 SA\right)^\frac{1}{4} < T/8 \\
    \iff &c_3^2 \beta_T^3 T^3 SA < T^4/8^4 \\
    \iff & T > 8^4 c_3^2 \beta_T^3 SA. 
\end{align*}
Defining $c_4\coloneq 8c_3$, for condition~\eqref{eq:lb_cor_cond_4} we compute
\begin{align*}
    & \beta_T SA \left( \sqrt{\frac{c_3 \beta_T T}{SA}} \right)^\alpha < T/8 \\
    \iff & T^\frac{2-\alpha}{2} > 8c_3^{\frac{\alpha}{2}}\beta_T^\frac{2+\alpha}{2} (SA)^\frac{2-\alpha}{2} \\
    \iff & T > SA 8 c_3^{\frac{\alpha}{2-\alpha}} \beta_T^{\frac{2+\alpha}{2-\alpha}}\\
    \impliedby & T> SA(c_4\beta_T)^{\frac{4}{2-\alpha}}
\end{align*}
where the final implication is due to $\alpha < 2$ and $c_3 > 1$.

So the desired contradiction holds as long as $T > \beta_T SA$, $T > 8^4 c_3^2\beta_T^3SA$, and $T > SA(c_4\beta_T)^{\frac{4}{2-\alpha}}$, but all of these conditions are clearly implied by $T > SA(c_4\beta_T)^{\frac{4}{2-\alpha}}$, since $\alpha \geq 1$ implies that $SA(c_4\beta_T)^{\frac{4}{2-\alpha}} \geq SA(8 c_3 \beta_T)^4$.
\end{proof}

\end{document}